\documentclass[12pt,final]{msml2020} 

\usepackage{bbm}

\usepackage{graphicx}
\usepackage{hyperref}
\usepackage{comment} 
\usepackage{color}
\usepackage{mathtools}
\usepackage{multicol}
\usepackage{footmisc}

\title[Gated RNNs]{Gating creates slow modes and controls phase-space complexity in GRUs and LSTMs}
\usepackage{times}

\newcommand*\samethanks[1][\value{footnote}]{\footnotemark[#1]}

\msmlauthor{
\Name{Tankut Can} \thanks{ TC \& KK contributed equally and listed
 alphabetically}\thanks{Corresponding authors: tankut.can@gmail.com and kameshk@princeton.edu } \\
 \addr Initiative for the Theoretical Sciences, CUNY Graduate Center
 \AND
 \Name{Kamesh {Krishnamurthy}} \samethanks[1]\samethanks[2]\\
 \addr Joseph Henry Laboratories of Physics and PNI, Princeton University
  \AND
 \Name{David J. {Schwab}} \\
 \addr Initiative for the Theoretical Sciences, CUNY Graduate Center
}

\begin{document}

\maketitle

\begin{abstract}%
Recurrent neural networks (RNNs) are powerful dynamical models for  data with complex temporal structure. 
However, training RNNs has traditionally proved challenging due to exploding or vanishing of gradients. RNN models such as LSTMs and GRUs (and their variants) significantly mitigate these 
issues associated with training by introducing various types of {\it gating} units into the architecture.
While these gates empirically improve performance, how the addition of gates influences the dynamics and trainability of GRUs and LSTMs is not well understood. Here, we take the perspective of studying randomly initialized LSTMs and GRUs as dynamical systems, and ask how the salient dynamical properties are shaped by the gates. 
We leverage tools from random matrix theory and mean-field theory to study the state-to-state Jacobians of GRUs and LSTMs. We show that the update gate in the GRU and the forget gate in the LSTM can lead to an accumulation of slow modes in the dynamics. Moreover, the GRU update gate  can poise the system at a marginally stable point. The reset gate in the GRU and the output and input gates in the LSTM control the spectral radius of the Jacobian, and the GRU reset gate also modulates the complexity of the landscape of fixed-points. Furthermore, for 
the GRU we obtain a phase diagram describing the statistical properties of fixed-points. We also provide a preliminary comparison of training performance to the various dynamical regimes realized by varying hyperparameters. Looking to the future, we have introduced a powerful set of techniques which can be adapted to a broad class of RNNs, to study the influence of various architectural choices on dynamics, and potentially motivate the principled discovery of novel architectures.

\end{abstract}

\begin{keywords}%
  RNN, GRU, LSTM, RMT, MFT
\end{keywords}

\section{Introduction}

Recurrent neural networks (RNNs) are a powerful class of dynamical systems that can implement a complex array of transformations between time-varying inputs and target outputs. These networks 
have proved to be highly effective tools in learning tasks involving data with complex temporal structure. However, RNNs in the form they were initially proposed 
are challenging to train due to the so-called exploding and vanishing gradients problem \cite{Bengio1994,hochreiter2001field,Pascanu2013}. 

More sophisticated RNN models such as long short-term memory networks (LSTMs)  \cite{Hochreiter1997} and gated recurrent units (GRUs) \cite{Cho2014} that feature
some form of gating exhibit significantly improved trainability. Empirically, these network models also achieve 
significant improvements over traditional RNNs in areas such
as language modeling \cite{Kiros2015}, speech recognition \cite{graves2013}, and
neural machine translation \cite{Cho2014,Sutskever2014,Bahdanau2014}. Thus gating seems
to robustly improve the performance of RNNs.

In addition to improved trainability, gating likely has a significant influence on the {\it dynamics}
of LSTMs and GRUs. However, the various architectural choices for the gates are 
somewhat {\it ad hoc}, and the precise nature of how the gates influence the dynamics 
and trainability of the network is not well understood. In this work, we take the approach of theoretically studying  LSTMs and GRUs 
 as dynamical systems and characterize how the gates shape the most salient aspects of the dynamics. 
 We use techniques from random matrix theory and mean-field theory to analytically characterize 
 the spectrum of the state-to-state Jacobian, and we study how each gate shapes the
features of the spectra. Our formalism properly accounts for the highly recurrent nature of 
 these systems which can lead to very different behavior from the feed-forward case.  In the GRU, the two gate types produce very distinct effects:
 the update gate can lead to a clumping of Jacobian eigenvalues near unity, facilitating a proliferation of long timescales and  {\it marginal stability}, whereas the reset gate influences the spectral radius and controls the topological complexity of phase space. We use the theory to develop a phase diagram for the GRU
 that characterizes the transition between different dynamical behaviors and the statistical properties of fixed points. For the LSTM, our spectral analysis of the Jacobian reveals the distinct roles of each gate: the forget gate primarily influences the accumulation of eigenvalues near unity, whereas all gates influence the spectral radius, though in differing degree. Finally, we provide a preliminary comparison of 
 training on a sequential task in the various dynamical regimes.


\section{Prior work}

Recent work \cite{Chen2018mft} has focused on characterizing the role of gating in conditioning
the Jacobian of a minimally gated model under the assumption that the weights are independent
from one time step to the other; this untied weight assumption is akin to studying a feed-forward gated network. The focus in 
that work was more on selecting good parameter initialization values for constraining the second 
moment of the Jacobian singular values, which improves trainability. 
This approach of using independent weights at each time step was later extended to LSTMs and GRUs \cite{Gilboa2019}. More recently \cite{Park2019arxiv} have looked at characterising GRUs as dynamical systems, but the focus is on small (two-dimensional) systems. \cite{tallec2018chrono} have argued that the update gates in GRUs can help with dealing with time warping. 

\cite{Lee2018,schoenholz2016deep,pennington2017resurrecting} have used random matrix theory (RMT) to characterize signal 
propagation in deep feedforward networks at initialization with independent weights in each layer. 
Seminal work by \cite{derrida1986random} and \cite{Sompolinsky1988a} used mean-field theory to study dynamics of randomly
connected recurrent networks without gating -- i.e. purely additive interactions. Subsequent work
\cite{Stern2014,Aljadeff2015chaos,Aljadeff2015evals,Marti2017, Ahmadian2015}  has extended this analysis using RMT to the case when the weight matrix has a  correlational structure. 

In the remainder of the paper, we first describe the gated networks we consider. Next we formulate
the problem of studying the spectrum of the state-to-state Jacobian as a random matrix problem. 
We then describe how the gates shape the short-term dynamics for GRUs and LSTMs. In Appendix \ref{app:train_seq_MNIST}, we provide preliminary results on training for a sequential 
task in the various dynamical regimes.

\section{Problem setup}

The `vanilla' RNN has no form of gating, and is described by a discrete time dynamical equation
\begin{align}
\mathbf{h}_{t}=\phi\left(U_{h} \mathbf{h}_{t-1}+ W_{x} \mathbf{x}_{t}+\mathbf{b}\right),
\end{align}
where $U_{h} \in \mathbb{R}^{N \times N}, W_{x} \in \mathbb{R}^{N \times N_{in}}, \mathbf{b} \in \mathbb{R}^{N}$,
$\mathbf{x}_t\in \mathbbm{R}^{N_{in}}$ is the input and $\mathbf{h}_t\in \mathbbm{R}^{N}$ is the internal (hidden) state of the RNN at the (integer) time step $t$. The nonlinear activation function $\phi$ is commonly taken to be $\tanh$, which will be our choice unless otherwise stated. 
These RNNs have been successful at learning sequential tasks but are notoriously hard to train due 
to the problem of exploding/vanishing gradients. Recently, LSTMs and GRUs were introduced to mitigate 
this issue by augmenting the vanilla RNN with gates that control information flow.  

For the GRU, in addition to the hidden state variables ${\bf h}_{t}$, there 
are two additional dynamical gating variables: an update gate ${\bf z}_{t}$ and a reset gate ${\bf r}_{t}$ which both take values ${\bf z}_{t}, {\bf r}_{t} \in (0,1)^{N}$. The dynamics 
of the GRU is given by 
\begin{align}
&{\bf z}_{t}   = \sigma \left(U_{z} {\bf h}_{t-1} + W_{z} {\bf x}_{t}+ {\bf b}_z\right), \quad {\rm update} \label{eq:gru_update}\\
&{\bf r}_{t}  = \sigma \left( U_{r} {\bf h}_{t-1}  + W_{r} {\bf x}_{t}+  {\bf b}_r\right), \quad {\rm reset}\label{eq:gru_reset}\\
&{\bf y}_{t}  =   U_h ( {\bf r}_{t} \odot {\bf h}_{t-1}) + W_{h} {\bf x}_{t}+ {\bf b}_{h},\\
&{\bf h}_{t} = {\bf z}_{t} \odot {\bf h}_{t-1} + (1 - {\bf z}_{t}) \odot \phi({\bf y}_{t}).	\label{eq:gru_eom}
\end{align}
where $\odot$ denotes element-wise product. 

Next, we consider LSTM networks introduced by  \cite{Hochreiter1997,Gers1999}. LSTMs have three 
gates (input ${\bf i}_{t}$, forget ${\bf f}_{t}$ and output ${\bf o}_{t}$), and their state is characterized by a cell state 
${\bf c}_{t}\in \mathbbm{R}^{N}$ and the hidden state ${\bf h}_{t}\in \mathbbm{R}^{N}$.  The LSTM dynamics is given by

\begin{multicols}{2}
\noindent
\begin{align*}
   \quad \quad  {\bf f}_t &= \sigma(  U_{f} {\bf h}_{t-1}+ W_{f} {\bf x}_{t} + {\bf b}_{f}) \quad {\rm forget}\\
   \quad \quad  {\bf o}_t & = \sigma(  U_{o} {\bf h}_{t-1} +W_{o} {\bf x}_{t}+ {\bf b}_{o}) \quad {\rm output}\\
{\bf c}_{t}  & = {\bf f}_{t} \odot {\bf c}_{t-1} + {\bf i}_{t} \odot \phi ( {\bf y}_{t}) 
\end{align*}
\begin{align}
    {\bf i}_t & = \sigma(  U_{i} {\bf h}_{t-1}+W_{i} {\bf x}_{t} + {\bf b}_{i}) \quad {\rm input} \label{eq:lstm_line1}\\
  {\bf y}_{t}  &= U_{h} {\bf h}_{t-1} +W_{h} {\bf x}_{t}+ {\bf b}_h \label{eq:lstm_line2}\\
   {\bf h}_{t} & = {\bf o}_{t} \odot \phi({\bf c}_{t}). 
\label{eq:lstm_eom}
\end{align}
\end{multicols}
We use the common choice of the sigmoid $\sigma(x) = (1 + e^{-x})^{-1}$ for the gating nonlinearity in both cases. 
We will study the dynamics of these models upon random parameter initialization -- i.e. the elements of the connectivity matrices and the bias vectors are assumed to be gaussian random variables, appropriately scaled by system size $(U_{k})_{ij}\sim \mathcal{N}(0, a_{k}^{2}/N)$, $({\bf b}_{k})_{i} \sim \mathcal{N}(b_{k}, v_{k})$, for $k \in \{ z, r, h\}$ (GRU) and $k \in \{ f, o, i,h\}$ (LSTM). In most of what follows, we assume the bias variance $v_{k} = 0$, unless otherwise explicitly stated. 

Before we proceed with the analysis, we note that just from these dynamical equations  one can understand simple intuitive features of each gate's effect.
For example, in the GRU the update gate (\ref{eq:gru_update}) controls the amount of {\it leak} and can thus potentially slow down the mixing of the inputs (as ${\bf z}_{t} \to 1$) due to self-coupling, thereby modulating memory in the network. The forget and input gate (\ref{eq:lstm_line1}) would appear to have analogous roles for the LSTM. The GRU reset gate (\ref{eq:gru_reset}) modulates (column-wise) the strength of the connectivity matrix $U_{h}$. One important consequence of this is that the reset gate has the power to change the landscape of dynamical fixed points, which is determined by $U_{h}$. The output gate (\ref{eq:lstm_line2}) appears to have an analogous role in modulating $U_{h}$ in the LSTM architecture. We will elaborate on these points below and provide quantitative understanding to these qualitative features. 

 Finally, we do not directly consider the effects of the input. However, nonzero bias is equivalent to the case in which the input is {\it constant} in time. We comment on the effects of the bias throughout the text.

\paragraph{Notation} 
We use ${\bf \hat{v}}$ to denote the diagonal matrix whose entries are given by the elements of the vector ${\bf v}$. The gate vectors are denoted ${\bf k}_{t} = \sigma( {\bf x}_{k})$ for ${\rm k} \in \{ z, r\}$ (GRU) and ${\rm k} \in \{ f, o, i\}$ (LSTM), and ${\bf x}_{k}$ is the appropriate argument of the sigmoid defined for GRU in (\ref{eq:gru_update}- \ref{eq:gru_reset}) and for LSTM in (\ref{eq:lstm_line1}- \ref{eq:lstm_line2}).
We use prime to denote differentiation, with ${\bf k}_{t}' = \sigma'({\bf x}_{k})$, where $\sigma'(x) = \sigma(x)(1 - \sigma(x))$, and similarly for $\phi'({\bf x})$. We denote the autocorrelation functions as $C_{x}(t, t') = \mathbbm{E}\left[ x_{t} x_{t'}\right]$, with the subscript indicating the variable and arguments $(t,t')$ indicating time. For quenched $U_k$, the expectation is understood as an average over neurons. In the mean-field limit, this is equivalent to an average over the effective stochastic variables. Finally, the asymptotic notation $O(\cdot)$ and $\Theta(\cdot)$ have the standard definitions.

\section{Spectral theory for the GRU}

In this section, we focus on the GRU. First, we develop a mean-field theory (MFT) for the GRU and calculate the state-to-state Jacobian with parameters drawn 
at initialization. The Jacobians are an architecture-dependent combination 
of structured and random matrices, and thus we can use tools from random matrix theory, combined with the MFT, to 
elucidate how the spectrum of the Jacobian is shaped by the gates.
The eigenvalues of the Jacobian will in principle depend on the particular realization of the connectivity matrices. However, for large random networks, due to self-averaging, the spectrum in fact only depends on the statistics of the state variables, and we  develop a self-consistent mean-field theory taking into account recurrent dynamics to calculate these statistics.

\subsection{Mean-field theory for the GRU}
We first develop a mean-field theory (MFT) for the GRU 
which gives valuable insight into the structure of the dynamical phase space as the model parameters are varied. First, the theory provides insight into the structure of correlation functions, which is required in the calculation of the instantaneous Jacobian spectrum. Second, it is useful in analyzing the structure of fixed-points (FPs) of the dynamics. 
The MFT replaces the $N$ dimensional update equation for the GRU  with a single stochastic difference equation. Specifically, we can approximate the following terms as Gaussian processes: 
\begin{align}
\left(U_{z} {\bf h}_{t} + {\bf b}_{z}\right)_{i} \sim \zeta_{t}, \quad
\left(U_{r} {\bf h}_{t} + {\bf b}_{r}\right)_{i} \sim \xi_{t}, \quad 
\left(U_{h} ( {\bf r}_{t}\odot {\bf h}_{t-1} ) + {\bf b}_{h}\right)_{i} \sim y_{t},
\end{align}
where $\zeta_{t}$, $\xi_{t}$ and $y_{t}$  are independent Gaussian processes specified
by their correlation functions which must be solved for self-consistently. 
From this, one can obtain a deterministic evolution equation for the correlation functions
of various dynamical variables, e.g. $C_{h}(t,t'):= \mathbbm{E}[ h_{t} h_{t'}]$. The details of the derivation
and the equations for the correlation functions in the general case are 
provided in Appendix \ref{app:gru-dmft}, and here, as an example, we  provide a summary 
of the MFT results for FPs (i.e. time-independent solution of the dynamics).

 At a FP, the MFT solution to the correlation functions are given by the following implicit equations: 
\begin{align}
C_{h} &=\int Dx \left(\phi(x\sqrt{a_{h}^{2}C_{y}+ v_{h}} + b_{h})\right)^{2},\quad C_{y} = C_{h} \int Dx \left(\sigma(x\sqrt{a_{r}^{2}C_{h} + v_{r}}  + b_{r})\right)^{2},\label{eq:FP1}
\end{align}
where $Dx = dx \exp(- x^{2}/2)/\sqrt{2\pi}$ is the Gaussian measure. We use perturbation theory to find solutions in Sec.(\ref{sec:fp-phase}), and map out a phase diagram in Fig.(\ref{fig:fp-phase}) indicating topologicaly distinct regions of phase space. 

Note that the update gate does not influence the fixed point solutions. This fact is apparent even at the level of  Eq. (\ref{eq:gru_eom}). However, as we shall see, the update gate can strongly 
affect the response to perturbations around the fixed points. 
Unlike the fixed point solutions above, the MFT solution for the correlation function in a general time-dependent state are obtained by solving a difference equation (see Appendix \ref{app:gru-dmft}). We will use the 
solutions for the correlation functions from the MFT in our RMT analysis of the state-to-state Jacobian.

\subsection{Spectral support for the GRU Jacobian}
The spectrum of the Jacobian provides valuable insight into
the instantaneous dynamics of the network. We wish to delineate the role of gates by determining how the choice of gates and parameters shapes the instantaneous Jacobian spectrum. To this end, we study the eigenvalues of the Jacobian of randomly initialized GRUs.
Our first result is a formula expressing the boundary of the spectral support, which we refer to as the {\it spectral curve}, in terms of expected values of functions of the dynamical state variables. 
From a theoretical perspective, this allows us to ascertain precisely how various architectural choices affect the spectrum. We use the method of hermitian reduction \cite{feinberg1997non} combined with the linearization trick from free probability theory \cite{Belinschi2018}, which are necessary ingredients to study the highly structured random non-Hermitian Jacobian. Our implementation of these techniques to the case of GRU and LSTM Jacobian spectra is described in detail in Appendix (\ref{app:gru-rmt}).

The state-to-state Jacobian for the GRU is given by
\begin{align}
{\bf J}_{t} = \frac{\partial {\bf h}_{t}}{\partial {\bf h}_{t-1}}	 = {\bf \hat{z}}_{t} + (\mathbbm{1}_{N} - {\bf \hat{z}}_{t}) \hat{\phi}'({\bf y}_{t}) U_{h}\left(  {\bf \hat{r}}_{t} +  {\bf \hat{h}}_{t-1} {\bf \hat{r}'}_{t} U_{r}\right) + \left({\bf \hat{h}}_{t-1} - \hat{\phi}({\bf y}_{t})\right) {\bf \hat{z}'}_{t} U_{z}.\label{eq:jac}
\end{align}

To study the spectrum of the Jacobian using random matrix techniques, the starting point 
is the resolvent $(\lambda \mathbbm{1}_{N} - {\bf J}_t)^{-1}$. From this, the 
spectral density is given by 
\begin{align}
\mu(\lambda) =\frac{1}{\pi} \frac{\partial}{\partial \bar{\lambda}}
\mathbbm{E}\left[  \frac{1}{N}\operatorname{tr}\left[(\lambda \mathbbm{1}_{N}-{\bf J}_t)^{-1}\right]  \right], \label{eq:density_resolvent}
\end{align}
where the expectation is over the random weight matrices. The analysis proceeds by considering the Jacobian for an enlarged space $({\bf h}_{t}, {\bf r}_{t}, {\bf z}_{t})$, 
\begin{align}
{\bf M}_{t} = \left( \begin{array}{ccc}
 {\bf \hat{z}}_t + (\mathbbm{1}_{N} - {\bf \hat z}_{t} ) \hat\phi'({\bf y}_{t}) U_{h} {\bf \hat r}_{t} & (\mathbbm{1} - {\bf \hat z}_{t} ) \hat\phi'({\bf y}_{t}) U_{h} {\bf \hat h}_{t-1} & {\bf \hat h}_{t-1} - \hat\phi({\bf y}_{t})\\
 {\bf \hat{r}}_{t}' U_{r} & 0 & 0\\
 {\bf \hat z}_{t}' U_{z} & 0 & 0
 \end{array}\right).	\label{eq:jac-M}
\end{align}

More precisely, ${\bf M}_{t}$ describes the linear dynamics:  $(\delta {\bf h}_{t+1}, \delta {\bf r}_{t}, \delta {\bf z}_{t})^{T} = {\bf M}_{t} (\delta {\bf h}_{t}, \delta {\bf r}_{t}, \delta {\bf z}_{t})^{T}$. The benefit of working with this representation is that each block is now linear in the Gaussian random weight matrices.\footnote{One might worry that, given the recurrent nature of the dynamics of the system of equations \ref{eq:gru_eom}, all the state variables will depend in some nonlinear fashion on the connectivity matrices, thus making ${\bf M}_{t}$ a complicated nonlinear function of the random variables. Fortunately, we can assume that  ${\bf h}_{t}$ and $U_{k}$ are independent. This assumption was initially made by \cite{Amari1972} and has been referred to as the ``local chaos hypothesis" \cite{Cessac1995} (see also \cite{Geman1982,Geman1982a}), where it was shown to hold for vanilla RNNs with asymmetric connectivity matrices.  Formally, this behavior may be established by the central limit theorem \cite{Geman1982,Geman1982a}. Our numerical experiments confirm that this assumption is borne out in the systems we study. 
} The only price we pay is that the eigenvalues of ${\bf J}_{t}$ must be found from a {\it generalized} eigenvalue problem ${\bf M}_{t} {\bf v} = {\bf I}_{\lambda} {\bf v}$ with the block diagonal matrix ${\bf I}_{\lambda} = {\rm bdiag} ( \lambda \mathbbm{1}_{N}, \mathbbm{1}_{N}, \mathbbm{1}_{N})$. The resolvent for this expanded Jacobian is given by 
\begin{align}
{\bf G}(\lambda) = \left( {\bf I}_{\lambda} - {\bf M}_{t}\right)^{-1},
\end{align}
which will be a $3N \times 3N$ matrix, whose first $N\times N$ block $(\lambda \mathbbm{1}_{N} - {\bf J}_{t})^{-1} = {\bf G}_{11}$ is the resolvent of the Jacobian, the object of interest. Since ${\bf J}_{t}$ is non-hermitian, we use the method of hermitian reduction to study the spectrum (details are in Appendix (\ref{app:gru-rmt}) ).

A major result of this analysis is  an expression for the spectral support of ${\bf J}_{t}$   in terms of 
the parameters $a_z,a_r,a_h$ and correlation functions of the dynamical variables in the mean field theory $(z_t=\sigma(\zeta_{t-1}),r_t = \sigma(\xi_{t-1}), h_{t-1},y_t)$, which we state below:

\begin{theorem}[ Spectral Support for GRU]\label{thm:GRU-spec-curve}
The support of the eigenvalue distribution of ${\bf J}_{t}$ in the limit of large $N$ is given by
$ \Sigma({\bf J}_{t}) :=  \left\{ \lambda \in \mathbbm{C}:  \quad \mathcal{S}(\lambda) \ge 0\right\}$
where 
\begin{align}
\mathcal{S}(\lambda)  =  	\rho_{t}^{2}\, \mathbbm{E}\left[ \frac{(1 - z_{t})^{2}}{|\lambda - z_{t}|^{2}} \right] 
& + a_{z}^{2}\,\mathbbm{E}\left[ \frac{(z_{t}')^{2}(h_{t-1} - \phi(y_t))^{2}}{|\lambda - z_{t}|^{2}} 
\right]-1 , \label{eq:spec-curve} \\ 
{\rm and  } \quad \rho_{t}^{2} = a_{h}^{2} C_{\phi'}(t, t) & \left[ C_{r}(t, t) + a_{r}^{2} C_{r'}(t,t) C_{h}(t-1,t-1)\right], 
\label{eq:spec-rad-shape-param}
\end{align}
and the boundary of the support defines the  spectral curve \: \:
$ \partial \Sigma({\bf J}_{t}):= \left\{ \lambda \in \mathbbm{C}: \quad \mathcal{S}(\lambda) = 0\right\}$.
\end{theorem} \label{def:rho}

The equation for the spectral curve Eq.( \ref{eq:spec-curve}) involves equal-time correlation functions 
of various dynamical variables, and to obtain these values accurately we resort to our 
self-consistent MFT for the correlation functions. Note that assuming independent
weights at each time step  will give erroneous results. 
Fig. (\ref{fig:gru-spectrum} a - d) shows the empirical spectrum along
with the spectral curve (in red) for various values
of $a_z$ and $a_r$.

One simple case that is nonetheless insightful is the zero fixed-point. The Jacobian around the zero fixed-point can
be analysed to yield the density of eigenvalues:
\begin{proposition}[Eigenvalue density for zero fixed point]\label{prop:gru-zero-FP}
Around the zero fixed point ($h_{t} = 0$), $z_{t} = z = \sigma(b_{z})$, $r_{t} = r = \sigma(b_{r})$, and the eigenvalue density is
\begin{align}
\mu(\lambda) = \frac{1}{\pi (1 - z)^{2} r^{2} a_{h}^{2}}, \quad {\rm for}\,\,\, |\lambda - z|\le (1 - z) r a_{h},
\end{align}
and zero otherwise. 	
\end{proposition}

Therefore, for zero activity networks, the Jacobian spectrum is uniform and occupies a circular region in the complex plane. In this setting, the consequences of gating are minimal. To observe nontrivial shaping of the spectrum by the gates, one must tune parameters outside the region in which the zero fixed point is stable. In fact, the zero fixed point becomes unstable precisely when the spectral radius exceeds unity. This condition follows directly from Prop. (\ref{prop:gru-zero-FP})

\begin{corollary}\label{corr:fp-gru}
The zero fixed point is {\bf stable} for
	\begin{align}
		\sigma(b_{z}) + (1 - \sigma(b_{z})) \sigma(b_{r}) a_{h}<1.
	\end{align}
\end{corollary}

With zero bias, this condition reduces to that found in \cite{Kanai2017} for stability of the zero fixed point.


\begin{figure}[h]
\begin{centering}
\includegraphics[scale=.8]{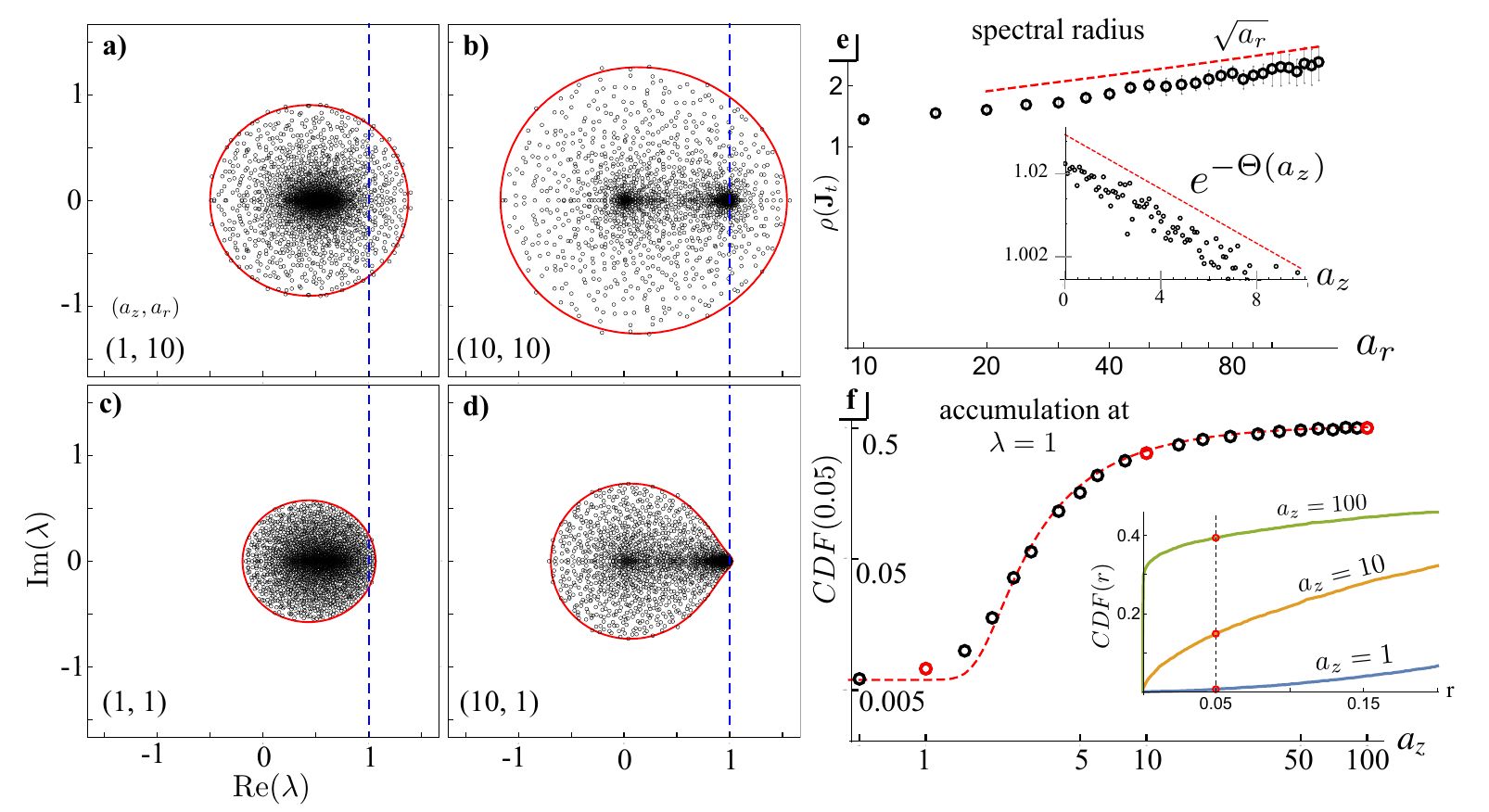}
\par \end{centering}
\caption{Empirical GRU Jacobian spectrum (black circles) with the spectral curve (\ref{eq:spec-curve}) predicted by RMT (red) in the steady state for $a_{h} = 3$ at various combinations of $(a_{z}, a_{r}) $ given by {\bf a)} $(1,10)$; {\bf b)} $(10,10)$; {\bf c)} $(1,1)$; {\bf d)}  $(10,1)$. In {\bf e)} we show the $\sqrt{a_{r}}$ growth of the spectral radius with increasing $a_{r}$ (for $a_{z} = 0$), as well as the exponential decrease (to unity) with increasing $a_{z}$ (for $a_{r} = 0$) (inset). In {\bf f)}, we show the cumulative distribution function ${\rm CDF}(r) = P(|\lambda - 1|<r)$ (inset) for various $a_{z}$. Increasing $a_{z}$ causes more eigenvalues to accumulate at $\lambda = 1$ as illustrated in the main figure which shows the CDF at $r = 0.05$. The red dashed curve is the scaling function (\ref{eq:cdf_scaling}).} \label{fig:gru-spectrum}
\end{figure}

\subsection{Shaping of dynamics by the {\it update}
and {\it reset} gates}
Having provided a method to calculate the spectral 
curve, we now look at how the two 
gates in the GRU -- the update and reset gates --
each shape the spectrum. For simplifying the discourse,
we consider either gate by itself to isolate its 
contribution, but this is not requisite. 

\subsubsection{The {\it update} gate facilitates creation of slow modes}
 Here we consider a GRU with only the update gate, and the reset gate variance set to $a_{r} = 0$. Eigenvalues 
 of the Jacobian close to $1.0$ correspond to 
 modes which will evolve slowly, and consequently
 a clumping of eigenvalues near $1.0$ will give
 rise to a broad spectrum of timescales. 
 In the limit $a_{z} = 0$, the eigenvalue density becomes circularly symmetric and centered at $\sigma(b_{z})$, similar to that of
 a vanilla RNN. In the opposite limit of $a_{z} = \infty$, the eigenvalues accumulate near unity and a characterisation of the full density  
 is easily illustrated by considering the nonlinear
 activation function $\phi$ to be piecewise-linear
 (i.e. the ``hard-tanh" defined in Eq.(\ref{eq:piecewise})). The behavior does not qualitatively change with other saturating non-linearities but the expressions can be complicated. Under this approximation, we can evaluate the density of eigenvalues $\mu(\lambda)$ for $a_z \to \infty$ and $b_{z}/a_{z} = \beta$ constant
\begin{align}
\mu(\lambda)	=  (1-\alpha)\delta(\lambda - 1) + \alpha(1 - \eta) \delta(\lambda) + \frac{1}{\pi a_{h}^{2}\sigma(b_{r})^{2}} {\bf 1}_{|\lambda|\le R},\label{eq:den-update-only} 
\end{align}
where $R = \sqrt{\alpha \eta}a_{h} \sigma(b_{r})$, $\eta$ is the fraction of unsaturated activations, and $\alpha$ is the fraction of update gates which are zero (for a  derivation of Eq. (\ref{eq:den-update-only}), see Appendix \ref{app:only-update}).  This fraction will change depending on $b_{z} = a_{z} \beta$, and in general will lead to an extensive number of eigenvalues at $\lambda = 1$. We show this accumulation at $\lambda = 1$ for the Jacobian of a generic time-dependent steady state in Fig.(\ref{fig:gru-spectrum} f), which shows good agreement with a scaling function : $c_{1} {\rm erfc}(c_{2}/a_{z})$ (dashed red); the motivation for this scaling function appears in Appendix \ref{app:only-update}. 
Interestingly, large $a_z$ also leads to 
{\it pinching} of the spectral curve (for $a_h$ in a certain range), thus further accentuating the accumulation of eigenvalues
near 1 (Fig. \ref{fig:gru-spectrum} f) and \ref{fig:fp-phase}b.
We will discuss pinching in greater detail in the context of marginal stability (Sec. \ref{sec:marginal-stability}).  The emergence of slow modes with a spectrum of timescales
is likely useful for processing inputs which have dependencies over a wide range of timescales. 
The accumulation of these slow modes is a generic feature of the gates that control the rate of integration, and we will show later that the forget gate in the LSTM has 
similar characteristics.

\subsubsection{The {\it reset} gate controls the spectral
radius}
The {\it reset} gate in the GRU has the effect of controlling the spectral radius $\rho({\bf J}_{t})$ of the Jacobian, defined $\rho({\bf J}_{t}) = {\rm max}\{ |\lambda_{i}|\}$.  The spectral radius among other things informs us about stability of fixed points. Furthermore, we argue that the reset gate  controls a transition from a topologically trivial dynamical phase space with only a single zero fixed point, to a ``complex" landscape with a large number of non-trivial fixed points \cite{wainrib2013}. To isolate the effect of the reset gate, we set the variance
of the update gate $a_z=0$. A consequence of Theorem \ref{thm:GRU-spec-curve} is the following corollary regarding the spectral radius:

\begin{corollary} For $a_{z} =0$, the spectral radius of the instantaneous Jacobian  $\rho({\bf J}_{t})$ is given by \begin{align}
\rho({\bf J}_{t}) = \sigma(b_{z}) + (1 - \sigma(b_{z})) \rho_{t},\end{align} 	
with $\rho_{t}$ defined in (\ref{eq:spec-rad-shape-param}). 
\end{corollary}

In the Appendix \ref{app:gru-reset-spectral-radius}, we give another proof of this using the Gelfand formula \cite{gelfand}, which indeed aligns with our RMT prediction. We clearly observe that higher values of $a_r$ lead to larger spectral radii, e.g. going from Figs. (\ref{fig:gru-spectrum}${\rm c} \to  {\rm a})$ or $(\ref{fig:gru-spectrum}{\rm d} \to {\rm b})$. Indeed, we can make this more precise with the following proposition:

\begin{proposition}\label{prop:reset-spectral-radius}
The spectral radius of the Jacobian ${\bf J}_{t}$ for $a_{z} = 0$ grows with $a_{r}$ asymptotically as $\rho({\bf J}_{t}) = \Theta ( \sqrt{a_{r}})$. 
\end{proposition}

The red dashed line in Fig (\ref{fig:gru-spectrum}e) demonstrates that this scaling takes over relatively soon.

\subsection{Phase diagram for the GRU}\label{sec:fp-phase}

The solutions of the MFT  allow us to map out dynamical regimes of qualitatively different behavior. This allows us to construct a phase diagram 
which, among other things, describes the statistical structure of fixed-points for the GRU. 
To illustrate our main point, we assume for simplicity $v_{h} = v_{r} = b_{h} = b_{r} = 0 $.

The zero solution $C_{h} = 0$ always exists, and is only stable for $a_{h}<2$ per Corollary \ref{corr:fp-gru}. A single (unstable) non-zero FP appears continuously from zero for $a_{h}>2$ (Fig. \ref{fig:fp-phase}). This motivates seeking a perturbative solution to Eq. (\ref{eq:FP1}) (details are in Appendix \ref{app:gru-fixed-point-mft}  ). We study this solution assuming small $C_h$, and $a_{h} = 2 + \epsilon$ for $\epsilon >0$ The perturbative FP solution can exhibit different behavior depending on the magnitude of $a_{r}$,
\begin{align}
C_{h} \sim  \begin{dcases}
 \frac{4 \epsilon}{8 - a_{r}^{2}} , \quad &a_{r}< \sqrt{2} a_{h},\\
 \frac{\sqrt{\epsilon}}{4}, \quad &a_{r} = \sqrt{2}a_{h}\\
 \frac{6 (a_{r}^{2} - 8)}{3 a_{r}^{4} + 24 a_{r}^{2} - 136} + \frac{f(a_{r}) \epsilon}{a_{r}^{2} - 8}, \quad &a_{r} > \sqrt{2} a_{h}
 \end{dcases}, \quad a_{h} = 2 + \epsilon.
\end{align}
where $f(a_{r})$ (given in (\ref{eq:fpoly})) is a rational polynomial which is positive and nonsingular for $a_{r} > \sqrt{8}$. For $a_{h} = 2 - \epsilon$ for positive $\epsilon$, there is no nonzero perturbative solution for $a_{r}\le \sqrt{2} a_{h}$, whereas for $a_{r}> \sqrt{2} a_{h}$ {\it two} nonzero fixed points arise (Fig. \ref{fig:fp-phase})
\begin{align}
C_{h} \sim \begin{dcases}
 	\frac{4 \epsilon}{a_{r}^{2} - 8}, \\
 	 \frac{6 (a_{r}^{2} - 8)}{3 a_{r}^{4} + 24 a_{r}^{2} - 136}- \frac{f(a_{r}) \epsilon}{a_{r}^{2} - 8}
 \end{dcases} \quad a_{r} > \sqrt{2} a_{h}, \quad a_{h} = 2 - \epsilon
\end{align}

Evidently, the first solution disappears at the critical value $a_{h} = 2$, while the second solution is continuous and nonzero across this threshold. 

Close to $a_{h} = 2^{-}$, we see that above an $a_{r}$-threshold, nonzero fixed points appear. They do not  exist for all $a_{h}<2$; in fact from Eq. \ref{eq:FP1} (for large $a_r$) we see that for $a_{h} < \sqrt{2}$, there are no non-zero fixed points, and for $a_{h} = \sqrt{2} + \epsilon$, a non-zero fixed point appears perturbatively, scaling like $C_{h} \sim \epsilon/\sqrt{2}$. Therefore, we see that the reset gate modulates the topology of the dynamical phase space in the interval $a_{h} \in (\sqrt{2}, 2)$, and produces nonzero fixed points in this regime where the zero FP is stable. In this range for $a_h$, increasing $a_{r}$ leads to a pitchfork-like bifurcation of the FP equations, from one (zero) FP to three FPs.  At a fixed $a_{h}$, this bifurcation occurs at a critical $a_{r}^{*}(a_{h})$ (with a correspoding variance $C_{h}^{*}$), which we refer to as the bifurcation curve, shown in Fig. (\ref{fig:fp-phase}, Left)  . 

Thus, the picture that emerges is the following (see Fig.\ref{fig:fp-phase} Left): when $a_h < \sqrt{2}$, the zero fixed-point is the only stable solution (blue region). For $\sqrt{2} < a_h < 2$, the zero FP is stable, but for $a_r$ above a critical value $a_r^*$, we also get
time-dependent chaotic solutions; this coincides with the abrupt appearance of the non-zero unstable FP solutions (in the pitchfork bifurcation) for $C_h$ in the MFT which indicate the proliferation of unstable fixed-points in the phase space (green region). When $a_h > 2$ the zero FP is unstable,  and the only solutions which are possible are time-dependent chaotic solutions (orange region). We emphasize that even though the appearance of chaotic dynamics coincides with the proliferation of unstable 
fixed-points, in this study we have only dealt with the analysis of fixed points. A detailed study of the chaotic solutions
will be published elsewhere.

\begin{figure}[h]
\begin{centering}
\includegraphics[scale=.75]{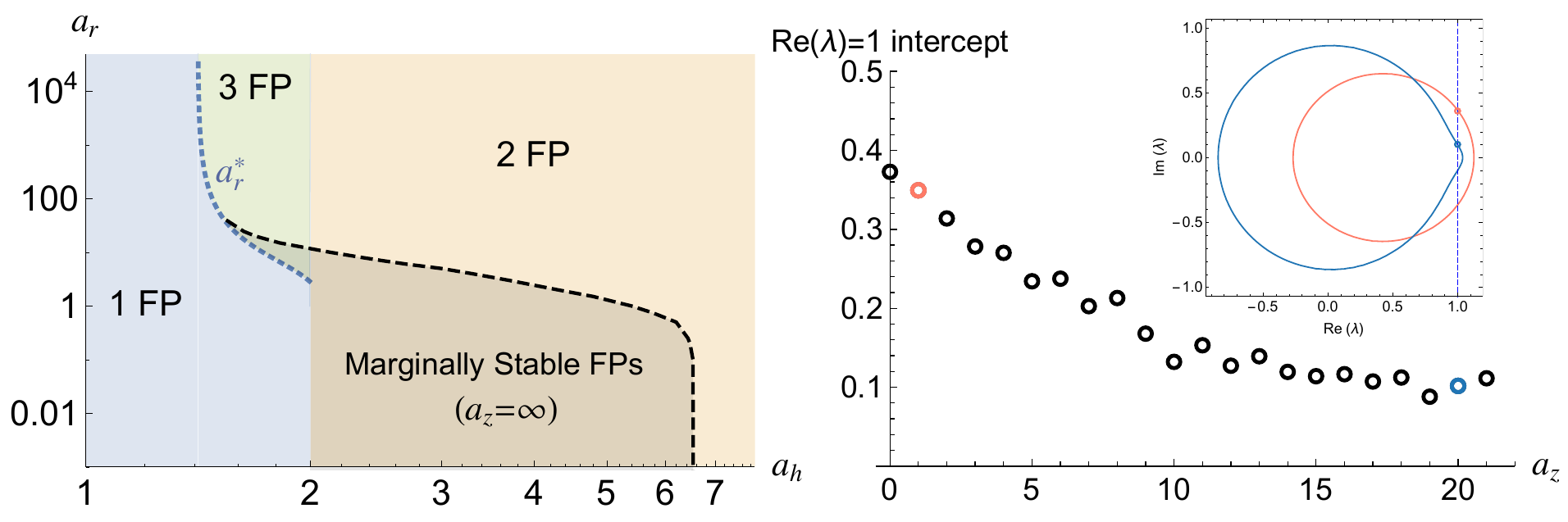}
\par\end{centering}
\caption{{\bf (Left)} Phase diagram in the $(a_{h},a_{r})$-plane for $b_{r} = 0$ indicating regions with distinct topological structure: (Blue) a single trivial (zero) fixed point only; (Green) three fixed points, two of which are non-zero; (Orange) two fixed points with a single non-zero FP. The bifurcation curve $a_{r}^{*}$ (dashed blue) separates the blue from the green region, and is computed numerically. Note that the zero FP always exists. ({\bf Right}) The pinching of the spectral support at $\lambda = 1$ is illustrated by tracking the intercept of the spectral boundary with ${\rm Re}(\lambda) = 1$. As this intercept decreases, the spectrum is seen (inset) to become pinched around $\lambda = 1$. Parameters used for numerical simulations: $N = 1000$, $a_{r} = 2$, $a_{h} = 3$.    } \label{fig:fp-phase}
\end{figure}

\subsection{Stability of non-zero fixed points and emergence of marginal stability} \label{sec:marginal-stability}

Having looked at the structure of non-zero fixed points in the phase diagram above, we now
comment on their stability and the ability of the update gate to make some of the fixed-points
marginally stable. The stability of fixed points can be determined from the spectral shaping parameter, and 
from Thm. (\ref{thm:GRU-spec-curve}) we can get a condition for when a fixed-point will be stable. Denoting the spectral shaping parameter (\ref{eq:spec-rad-shape-param}) at the fixed point $\rho_{0}$, we find the stability condition:

\begin{corollary}[Fixed Point Stability]\label{cor:fp-stability} A fixed point is stable iff the spectral shaping parameter (\ref{eq:spec-rad-shape-param}) satisfies $\rho_{0} <1$. 	
\end{corollary}
\begin{proof}
For a fixed point, the second term in (\ref{eq:spec-curve}) vanishes. We make use of two observations: that $\mathcal{S}(\lambda)$ is a monotonically decreasing function of $|\lambda|$, and that $\mathcal{S}(|\lambda|e^{i\theta}) \le \mathcal{S}(|\lambda|)$ for any nonzero $\theta$. Together these imply that the spectral radius is determined by the largest positive real $x$ such that $\mathcal{S}(x) = 0$. Thus, if $\mathcal{S}(1) <0$, then this implies $x<1$, and hence stability. Conversely, if the FP is stable, $x<1$ which implies $\mathcal{S}(1) < 0$ since $\lambda = 1$ is outside the support. Finally, evaluating $\mathcal{S}(1) = \rho_{0}^{2} - 1$, concludes the proof.
\end{proof}

It turns out that for all non-zero fixed points in the green and orange region in Fig. (\ref{fig:fp-phase}), which satisfy the implicit mean-field equations (\ref{eq:FP1}), we find $\rho_{t} >1$. Thus, all of the nonzero fixed points implied by the mean-field theory are {\it unstable}. 
However, the locally unstable manifold of these fixed-points can be reduced significantly by an appropriate scaling of the update gate variance $a_{z}$, and in the asymptotic limit, some
of these fixed-points become marginally stable. We state this result below:

\begin{theorem}[Emergence of Local Marginal Stability]\label{thm:pinching}
	For large $a_{z}$, and $b_{z} = a_{z} \beta$, the asymptotic scaling of the spectral radius of the fixed point Jacobian ${\bf J}_{0}$ is given by
	
	\begin{align}
	 \rho({\bf J}_{0}) - 1 = \Theta( e^{ - c \sqrt{C_{h}} a_{z}})	,
	\end{align}
where $c \in \mathbbm{R}^{+}$ is determined implicitly by 
\begin{align}
{\rm erf} \left( \frac{c - \beta/\sqrt{C_{h}}}{\sqrt{2}}\right) = \frac{2}{\rho_{0}^{2}} - 1, \label{eq:c}
\end{align}
and $\rho_{0}$ is given by Eq.(\ref{eq:spec-rad-shape-param}) evaluated at the fixed point.
\end{theorem}

Thus for fixed-points which admit a positive solution for $c$, we see that increasing $a_{z}$ makes the spectral radius approach unity, and the fixed-points become marginally stable. Moreover, 
on increasing $a_z$ not only does the spectral radius reduce to unity for these fixed-points, but
the spectral density also gets {\it pinched} near unity leading to a concentration of eigenvalues 
near unity (Fig. \ref{fig:fp-phase} right). In appendix (\ref{app:pinching}), we show that the leading edge of the spectral boundary curve is contained in a wedge near $\lambda = 1$ which scales like $\lambda_{0} \exp \left( - c \sqrt{C_{h}} a_{z}\right)$, for some order one constant $\lambda_{0} \in \mathbbm{C}$. 

We now comment on the condition for a fixed-point to become marginally stable: a positive solution for $c$  in (\ref{eq:c}) only exists for 

\begin{align}
\rho_{0}^{2} < \frac{1}{\alpha}, \quad \alpha = \frac{1}{2} {\rm erfc} (\beta/\sqrt{2 C_{h}}).\label{eq:alpha}
\end{align}

In the $a_{z} \rightarrow \infty$ limit, $\alpha = 1-\mathbbm{E}\left[ z_{t}\right]$ is simply the fraction of gates which are completely closed ($z_{t} = 0$). When the bias is zero, $\alpha = 1/2$, and the condition for the spectral shaping parameter becomes $\rho_{0} < \sqrt{2}$. 

Thus, while $\rho_{0}>1$ for all nonzero fixed points, there is a range of parameters for which $\rho_{0} < \sqrt{2}$ and marginal stability and spectral pinching can occur for nonzero fixed points. This region in parameter space is shaded black in Fig. (\ref{fig:fp-phase} Left). This pinching effect can be observed even for finite $a_{z}$ (see Fig. (\ref{fig:gru-spectrum} d) for $(a_{z}, a_{r}) = (10,1)$, which falls inside the marginal stability region). A practical consequence of the pinching is that it will reduce the number of unstable modes, leading to a non-negligible chance of observing nonzero fixed points in finite random networks.


\section{Spectral theory of LSTM}

The analysis of the Jacobian for the LSTM proceeds in a similar way as the GRU. For an LSTM with $N$ hidden and cell state variables, the Jacobian is the  $2N\times 2N$ matrix

\begin{align}
{\bf J}_{t} &= \left( \begin{array}{cc}
 	{\bf \hat f}_{t} & {\bf \hat g}_t\\
 	 {\bf  \hat m}_{t} \, {\bf  \hat f}_{t}   & {\bf \hat o}_{t}'\, \hat\phi({\bf c}_{t}) U_{o} + {\bf \hat m}_{t} {\bf \hat g}_{t}
 \end{array}\right)	, \quad {\bf g}_{t} = {\bf \hat f}_{t}'\,{\bf \hat c}_{t-1} \, U_{f}+{\bf \hat i}_{t} \,\hat\phi'({\bf y}_{t})\, U_{h} + {\bf \hat i}_{t}'\,\hat\phi({\bf y}_{t})  U_{i}, \label{eq:lstm-jacobian}
\end{align}
where ${\bf m}_{t} = {\bf o}_{t} \odot \phi'({\bf c}_{t})	$.
As with the GRU, we wish to evaluate the LSTM resolvent, and thus characterize the spectrum, in the limit of large $N$. We accomplish this by again combining linearization with hermitian reduction, then performing the ensemble average over $U_{k}$ (for details see Appendix (\ref{app:lstm-rmt})). Anything we do not average over will enter into the final expression. However, as we demonstrate in the appendix, these expressions involve averages of state variables over the network. To state our results below, we assume the mean-field limit to simplify the resulting correlation functions. In this limit, $\left(U_{k} {\bf h}_{t-1} + {\bf b}_{k}\right)_{i} \sim \eta_{t-1}^{k}$ and $(U_{h} {\bf h}_{t-1} + {\bf b}_{h})_{i} \sim y_{t}$ become independent Gaussian processes, and $c_{t},h_{t}$ are random variables (for details see App. (\ref{app:gru-dmft})). For the gating variables we use the notation $k_{t} = \sigma(\eta_{t-1}^{k})$, and $k_{t}' = \sigma'(\eta_{t-1}^{k})$. Our first result concerns the trace of the resolvent of ${\bf J}_{t}$ in the large $N$ limit in terms of the parameters $a_{f},a_i,a_o,a_h$, as well as the spectral support. Note that we normalize the resolvent by $N$ as in (\ref{eq:density_resolvent}) instead of $2N$.

\begin{theorem}[LSTM Resolvent]\label{thm:LSTM-spec-curve}
The trace of the resolvent for the LSTM Jacobian in the large $N$ limit in the mean-field theory is 
\begin{align}
G(\lambda) &= \begin{dcases}
 	\mathbbm{E}\left[ \frac{|\lambda|^{2}(\bar{\lambda} - f_{t}) + \bar{\lambda}|\lambda - f_{t}|^{2} + F(\lambda) \left( (\bar{\lambda} - f_{t}) q_{t} + \bar{\lambda} p_{t}\right)}{|\lambda|^{2} |\lambda - f_{t}|^{2} + F(\lambda) \left( |\lambda - f_{t}|^{2} q_{t} + |\lambda|^{2} p_{t}\right)} \right]	, \quad &\lambda \in \Sigma,\\
\frac{1}{\lambda}+ \mathbbm{E}\left[  \frac{1}{\lambda - f_{t}} \right] 	, \quad &\lambda \in \Sigma^{c},\\
 \end{dcases} \label{eq:lstm-G}
\end{align}
where 
\begin{align}
  q_{t} = a_{o}^{2} o_{t}'^{2} \phi(c_{t})^{2}, \qquad  p_{t}  &=  o_{t}^{2} \phi'( c_{t})^{2} \left( a_{f}^{2} c_{t-1}^{2}  f_{t}'^{2} + a_{i}^{2} i_{t}'^{2} \phi(y_{t})^{2} + a_{h}^{2} i_{t}^{2} \phi'( y_{t})^{2}\right), \label{eq:lstm-shaping-var}
\end{align}
and $F(\lambda)\ge 0$ for $\lambda \in \Sigma$ solves the implicit equation

\begin{align}
	1 &= \mathbbm{E} \left[ \frac{|\lambda - f_{t}|^{2} q_{t} + |\lambda|^{2} p_{t}}{|\lambda|^{2}|\lambda - f_{t}|^{2} + F(\lambda) \left(|\lambda - f_{t}|^{2} q_{t} + |\lambda|^{2} p_{t}\right)}\right],	\label{eq:lstm-F}
\end{align}
whereas $F = 0$ outside the support $\lambda \in \Sigma^{c}$.

\end{theorem}

Outside the support, the resolvent becomes a holomorphic function of $\lambda$, whereas inside it depends on both $\lambda$ and $\bar{\lambda}$. The trace of the resolvent $G(\lambda)$ is continuous on the complex plane, which means the holomorphic solution must match the non-analytic solution precisely at the boundary of the support. This allows us to deduce the spectral boundary curve by setting $F(\lambda) = 0$ in (\ref{eq:lstm-F}).

\begin{corollary}[LSTM Boundary Curve]\label{cor:lstm-spec-curve}
The boundary of the spectral support is given by the curve
\begin{align}
\partial \Sigma({\bf J}_{t}) = \{ \lambda \in \mathbbm{C}: \mathcal{S}(\lambda) = 0 \}	\label{eq:lstm_spec_curve}
\end{align}
	where
	\begin{align}
	\mathcal{S}(\lambda) = 	\frac{\mathbbm{E}\left[ q_{t}\right]}{|\lambda|^{2}} + \mathbbm{E}\left[ \frac{p_{t}}{|\lambda - f_{t}|^{2}}\right]	-1.\label{eq:lstm_S}
	\end{align}

\end{corollary}


\begin{figure}[h]
\begin{centering}
\includegraphics[scale=0.9]{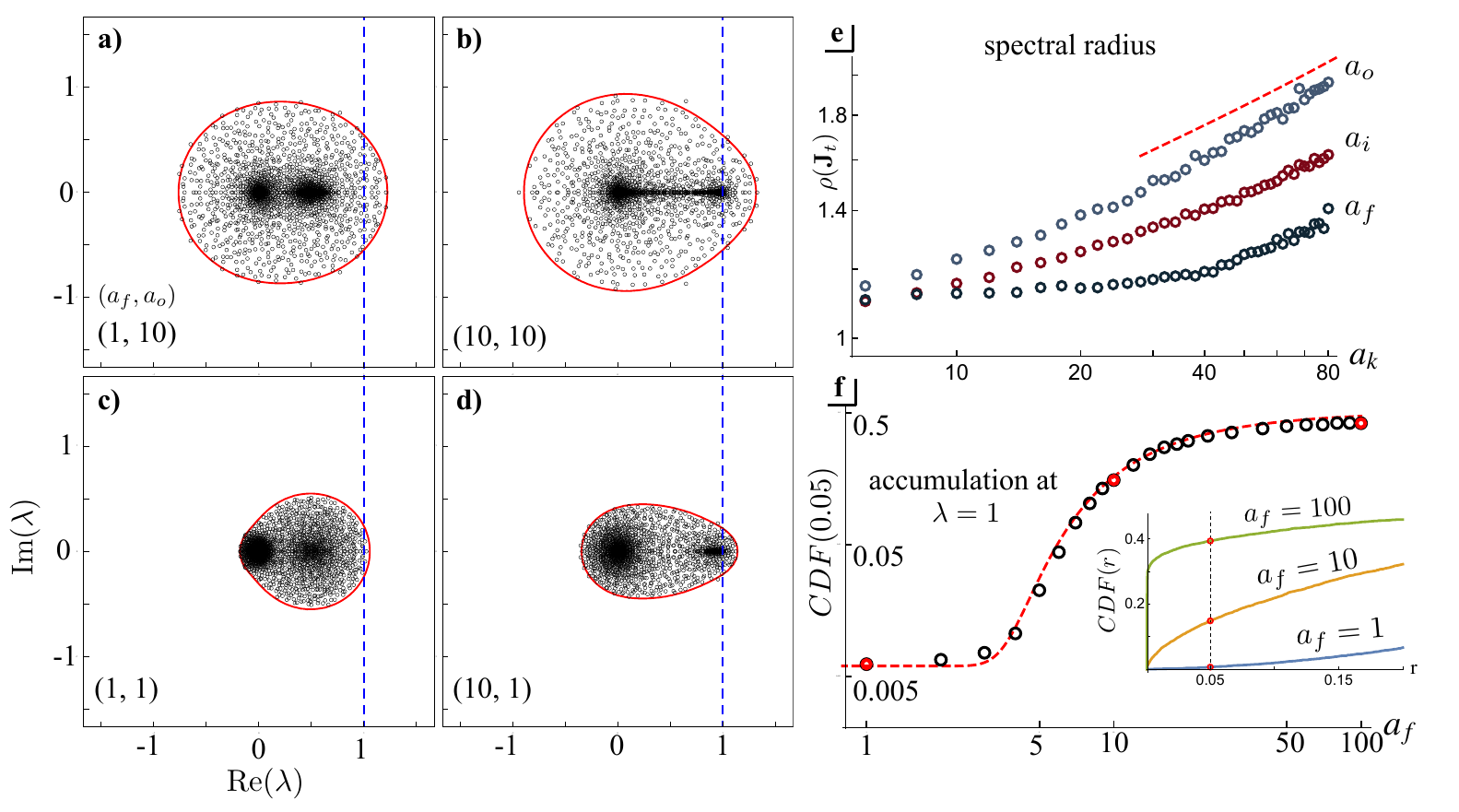}
\par\end{centering}
\caption{({\bf a - d}) Empirical spectrum of LSTM Jacobian in steady state (black circles) compared to RMT spectral boundary prediction (\ref{eq:lstm_spec_curve}) (red) using steady state correlation functions: $N = 1000$, $a_{h} = 3$, $a_{i} = 1$, and the values of $(a_{f},a_{o})$ are indicated for each plot. ({\bf e}) Numerically computed spectral radius as a function of $a_{k}$ when all other parameters ($a_{k'}$, $k' \ne k$) are set to zero, with $\sqrt{a_{k}}$ scaling (dashed red) to guide the eye. ({\bf f}) Cumulative distribution of eigenvalues around a small ball of radius $r = 0.05$ as a function of $a_{f}$, with $a_{i} = a_{o} = 0$ and zero biases, compared to scaling function (\ref{eq:cdf_scaling}) (dashed red). Inset shows CDF as a function for $r$ for different $a_{f}$ for $a_{i} = a_{o} = 0$} \label{fig:lstm-plots}
\end{figure}
 
We present the general result for the density of eigenvalues in the appendix, which is not very illuminating. However, around the zero fixed point, the density of eigenvalues simplifies considerably:

\begin{corollary}[Zero Fixed Point]\label{cor:lstm-zero-fp}
	Around the zero fixed point $c_{t} =  h_{t} = 0$, $ p_{t} = p =\sigma(b_{o})^{2} \sigma(b_{i})^{2} a_{h}^{2}$, $q_{t} = 0$, and the density of eigenvalues becomes
	
	\begin{align}
	\mu(\lambda) = \delta(\lambda) + \frac{1}{\pi p}, \quad {\rm for}\,\, \,|\lambda - f|\le \sqrt{p}	
	\end{align}
and zero otherwise.
	
\end{corollary}

Note with our normalization, $\int d \mu(\lambda) = 2$. From this, we can directly infer the stability condition for the zero fixed point

\begin{corollary}\label{cor:lstm-zero-fp-stability}
The zero fixed point is {\bf stable} for $$\sigma(b_{f}) + \sigma(b_{o})\sigma(b_{i}) a_{h} <1$$
\end{corollary}

It is worth emphasizing that the density of eigenvalues depends on most state variables (e.g. input and output gates, activation functions) only through the combinations which appear in (\ref{eq:lstm-shaping-var}). An exception is the forget gate, which enters separately and in a rather influential way. Consequently, we find that very naturally the effects of the input and output gates in shaping the spectrum are roughly comparable, while the forget gate plays a qualitatively distinct role.

\subsection{The {\it forget} gate facilitates slow modes}
The {\it forget} gate in the LSTM has a role very similar to that of the 
GRU {\it update} gate. In particular, large values of $a_f$
facilitate the accumulation of eigenvalues near unity thus leading to slow
modes (Fig. (\ref{fig:lstm-plots}f)).
 In the  limit $a_{f} = \infty$, using again a piecewise linear approximation of the activation functions $\phi$, we find the density of the eigenvalues 
\begin{align}
\mu(\lambda) = \left(1 + \frac{(1 - \eta)}{2}\right) \delta(\lambda) +	\frac{1}{2} \delta(\lambda - 1) + \frac{1}{\pi R^{2}} {\bf 1}_{\lambda \in \Sigma}, \quad \Sigma: = \{ \lambda \in \mathbbm{C}: |\lambda| \le \frac{\sqrt{\eta}}{\sqrt{2}} R \},\label{eq:lstm-density-binary-forget}
\end{align}
where $R = \sigma(b_{o}) \sigma(b_{i}) a_{h}$, and $\eta$ is again the fraction of activations which are unsaturated (see App. (\ref{app:lstm-only-forget})). As before, we see that ramping up $a_{f}$ leads to an accumulation of eigenvalues at $\lambda = 1$, becoming a delta function in the extreme limit. 

When $a_{o} = 0$, the first $N$ rows of ${\bf J}_{t}$ (\ref{eq:lstm-jacobian}) become linearly dependent with the second $N$ (due to $U_{0}$ vanishing), guaranteeing at least $N$ exact zero eigenvalues. The accumulation at zero will generally grow in proportion to the fraction of saturated activations. Finally, in the limit $a_{f} \rightarrow \infty$, the mean density away from these accumulation points is flat, controlled by the input and output biases as well as the variance $a_{h}$. 

An important assumption in arriving at this expression is that the forget gate is completely saturated, essentially becoming a binary variable. Practically speaking, this allows us to neglect contributions from $f'_{t}$. As a result, we find a spectral radius which is independent of $a_{f}$, fixed by other parameters in the problem. However, we observe numerically in Fig.(\ref{fig:lstm-plots}) that the spectral radius in fact grows with $a_{f}$, albeit slower than with $a_{i}$ and $a_{o}$. The scaling of the spectral radius in this regime is shown to be no greater than $\sqrt{a_{f}}$ in Appendix (\ref{app:lstm-only-forget}). 

\subsection{The {\it input} and {\it output} gates modulate the spectral radius}

The {\it output} and {\it input} gates in the LSTM have similar effects on the shape of the Jacobian and are comparable in this regard to the {\it reset} gate in the GRU; the effect of the {\it forget} gate on the spectral radius is modest (see Fig.(\ref{fig:lstm-plots}e). Setting $a_{f} = 0$, the spectral boundary curve simplifies considerably to read
\begin{align}
1 = \frac{\mathbbm{E}[q_{t}]}{|\lambda|^{2}} + \frac{\mathbbm{E}[ p_{t}]}{|\lambda - f|^{2}}.	
\end{align}
The spectral radius will consequently be controlled by the asymptotic behavior of the coefficients $\mathbbm{E}[q_{t}]$ and $\mathbbm{E}[p_{t}]$. First, we consider $a_{o} = 0$. In this case, it immediately follows that $q_{t} = 0$, and we are left with only $p_{t}$. The spectral support then becomes a circle centered on $\sigma(b_{f})$ with radius $\sqrt{ \mathbbm{E}[p_{t}]}$. In appendix (\ref{app:lstm-only-input}), we show that when $a_{o} = 0$, $\mathbbm{E}[p_{t}] = O(a_{i})$. 

Isolating the effects of the output gate by setting $a_{i} = 0$, we find that $\mathbbm{E}[p_{t}] = O(1)$, whereas $\mathbbm{E}[q_{t}] =  O(a_{o})$.

Gathering these results, we find a universal scaling of the spectral radius with the gate variance:

\begin{proposition}
The spectral radius of the LSTM Jacobian scales asymptotically as
\begin{align}
    \rho({\bf J}_{t}) = O(\sqrt{a_{k}}), \quad a_{k}\to \infty, \quad a_{k' \ne k} = {\rm fixed} 
\end{align}
\end{proposition}

We compare this scaling behavior with numerical experiments in Fig. (\ref{fig:lstm-plots}), showing good agreement. We give additional arguments in Appendix (\ref{app:lstm-only-input}) and (\ref{app:lstm-only-output}) that in fact $\rho({\bf J}_{t}) = \Theta(\sqrt{a_{k}})$ for the input and output gates $k = i, o$. This indicates the spectral radius indeed strictly grows with these gate variances. However, we lack a tight lower bound for the growth of the spectral radius with the forget gate, most likely because the lower bound is an order one constant, as suggested by the $a_{f} = \infty$ analysis. Nevertheless, we are able to observe the $\sqrt{a_{f}}$ scaling in Fig. (\ref{fig:lstm-plots}) over a modest range of $a_{f}$.


\section{Discussion: future directions
and relation to training}

We have undertaken a detailed analysis of the 
dynamics of randomly initialized GRUs and LSTMs. 
These dynamical properties are critical to how 
the networks  process temporal input.
We showed that the gates which affect the
rate of integration -- the update gate in the GRU and
the forget gate in the LSTM -- can faciliate long timescales
by shaping the Jacobian spectrum to cluster near unity.
Moreover, we showed how the update gate in the GRU, in
certain regimes, can make the system {\it marginally stable}.

Marginal stability can propagate perturbations over a spectrum of long
timescales and facilitate the existence of
a continuous manifold of attractors. Recently \cite{Niru2019} found that
trained gated-RNNs naturally discover solutions for which
the dynamics tends toward a line attractor.
Our analysis shows that marginal stability is a generic feature in
the presence of binary switch-like gates that control the rate of integration.
We showed that the reset gate in the GRU and the output and input gates
in the LSTM control the spectral radius of the Jacobian and
thus shape the complexity of the dynamics. For the GRU, we
obtained a full phase-diagram of dynamics and showed how the reset gate
can lead to a proliferation of unstable fixed points coinciding
with the abrupt appearance of a chaotic attractor. 

How does our study of the dynamical properties of gated RNNs
inform us about training performance? While the dynamical
properties are crucial for performance, these properties will change
during training;  however, initialization is known to have a significant impact on trainability \cite{Sutskever2012}. Furthermore, in feedforward networks, it was shown that parameters do not change significantly from initialization when training large networks with appropriately small learning rate \cite{jacot2018,Lee2018}. This regime is relatively unexplored for RNNs. A detailed study
of how the dynamical properties change during training will be
undertaken elsewhere, but in Appendix (\ref{app:train_seq_MNIST}) we provide very preliminary
results on training GRUs on a sequential task that requires
processing long-time dependencies. Overall, we observe (see Fig \ref{fig:training_fig_1})
that both training
time and training accuracy improve with larger values of $a_z$ (more
slow modes). Interestingly, we observe that values of $a_h$ slightly
above the critical point $a_h=2.0$  (in the marginally stable region indicated in Fig. (\ref{fig:fp-phase})) give the best performance (for high $a_z$).
This bears some semblance to the benefits of ``edge of chaos" initialization in vanilla RNNs \cite{Bertschinger2004}, and feed forward networks \cite{glorot2010understanding,schoenholz2016deep}. The improved performance with increasing $a_{z}$ furthermore suggests that proximity to marginal stability is also a boon for training. This might explain the relative ease of training GRUs and LSTMs, since, as we showed, gating in both architectures quite naturally leads to marginal stability.
A more thorough characterization of how dynamics evolve during training and the
relation to training performance will be undertaken in future work. 
We have provided analytical tools to study the dynamics of RNNs that we
expect can be fruitfully applied in contexts other than at initialization.

\paragraph{Acknowledgements} We would like to thank Devon Wood-Thomas for comments on the draft. KK is supported by a C.V. Starr Fellowship and a CPBF Fellowship (through NSF PHY-1734030). DJS acknowledges support from NSF PHY-1734030, NIH 5R01EB026943 and a Simons Investigator grant in MMLS.

\bibliography{Neuro.bib}

\begin{thebibliography}{41}
\providecommand{\natexlab}[1]{#1}
\providecommand{\url}[1]{\texttt{#1}}
\expandafter\ifx\csname urlstyle\endcsname\relax
  \providecommand{\doi}[1]{doi: #1}\else
  \providecommand{\doi}{doi: \begingroup \urlstyle{rm}\Url}\fi

\bibitem[Ahmadian et~al.(2015)Ahmadian, Fumarola, and Miller]{Ahmadian2015}
Yashar Ahmadian, Francesco Fumarola, and Kenneth~D Miller.
\newblock {Properties of networks with partially structured and partially
  random connectivity}.
\newblock \emph{Physical Review E}, 91:\penalty0 012820, 2015.

\bibitem[Aljadeff et~al.(2015{\natexlab{a}})Aljadeff, Renfrew, and
  Stern]{Aljadeff2015evals}
Johnatan Aljadeff, David Renfrew, and Merav Stern.
\newblock Eigenvalues of block structured asymmetric random matrices.
\newblock \emph{Journal of Mathematical Physics}, 56:\penalty0 103502,
  2015{\natexlab{a}}.

\bibitem[Aljadeff et~al.(2015{\natexlab{b}})Aljadeff, Stern, and
  Sharpee]{Aljadeff2015chaos}
Johnatan Aljadeff, Merav Stern, and Tatyana Sharpee.
\newblock Transition to chaos in random networks with cell-type-specific
  connectivity.
\newblock \emph{Physical Review Letters}, 114:\penalty0 088101,
  2015{\natexlab{b}}.

\bibitem[Amari(1972)]{Amari1972}
Shun-Ichi Amari.
\newblock Characteristics of random nets of analog neuron-like elements.
\newblock \emph{IEEE Transactions on Systems, Man, and Cybernetics}, \penalty0
  (5):\penalty0 643--657, 1972.

\bibitem[Bahdanau et~al.(2014)Bahdanau, Cho, and Bengio]{Bahdanau2014}
Dzmitry Bahdanau, Kyunghyun Cho, and Yoshua Bengio.
\newblock Neural machine translation by jointly learning to align and
  translate.
\newblock In \emph{International Conference on Learning Representations
  (ICLR)}, 2014.

\bibitem[Belinschi et~al.(2018)Belinschi, {\'{S}}niady, and
  Speicher]{Belinschi2018}
Serban~T. Belinschi, Piotr {\'{S}}niady, and Roland Speicher.
\newblock {Eigenvalues of non-Hermitian random matrices and Brown measure of
  non-normal operators: Hermitian reduction and linearization method}.
\newblock \emph{Linear Algebra and Its Applications}, 537:\penalty0 48, 2018.

\bibitem[Bengio et~al.(1994)Bengio, Simard, and Frasconi]{Bengio1994}
Yoshua Bengio, Patrice Simard, and Paolo Frasconi.
\newblock Learning long-term dependencies with gradient descent is difficult.
\newblock \emph{IEEE Transactions on Neural Networks}, 5\penalty0 (2):\penalty0
  157--166, 1994.

\bibitem[Bertschinger and Natschl(2004)]{Bertschinger2004}
Nils Bertschinger and Thomas Natschl.
\newblock {Real-Time Computation at the Edge of Chaos in Recurrent Neural
  Networks}.
\newblock \emph{Neural Computation}, 1436:\penalty0 1413--1436, 2004.

\bibitem[Cessac(1995)]{Cessac1995}
B~Cessac.
\newblock {Increase in Complexity in Random Neural Networks}.
\newblock \emph{J Phys I France}, 5\penalty0 (3):\penalty0 409--432, 1995.

\bibitem[Chen et~al.(2018)Chen, Pennington, and Schoenholz]{Chen2018mft}
M.~Chen, J.~Pennington, and S.~S. Schoenholz.
\newblock {Dynamical Isometry and a Mean Field Theory of RNNs: Gating Enables
  Signal Propagation in Recurrent Neural Networks}.
\newblock In \emph{International Conference on Machine Learning (ICML)}, 2018.

\bibitem[Cho et~al.(2014)Cho, van Merri{\"e}nboer, Gulcehre, Bahdanau,
  Bougares, Schwenk, and Bengio]{Cho2014}
Kyunghyun Cho, Bart van Merri{\"e}nboer, Caglar Gulcehre, Dzmitry Bahdanau,
  Fethi Bougares, Holger Schwenk, and Yoshua Bengio.
\newblock Learning phrase representations using {RNN} encoder{--}decoder for
  statistical machine translation.
\newblock In \emph{Proceedings of the 2014 Conference on Empirical Methods in
  Natural Language Processing ({EMNLP})}, pages 1724--1734, 2014.

\bibitem[Derrida and Pomeau(1986)]{derrida1986random}
Bernard Derrida and Yves Pomeau.
\newblock Random networks of automata: a simple annealed approximation.
\newblock \emph{EPL (Europhysics Letters)}, 1\penalty0 (2):\penalty0 45, 1986.

\bibitem[Feinberg and Zee(1997)]{feinberg1997non}
Joshua Feinberg and Anthony Zee.
\newblock Non-hermitian random matrix theory: Method of hermitian reduction.
\newblock \emph{Nuclear Physics B}, 504\penalty0 (3):\penalty0 579--608, 1997.

\bibitem[Gelfand(1941)]{gelfand}
I.~Gelfand.
\newblock Normierte ringe.
\newblock \emph{Rec. Math. [Mat. Sbornik] N.S.}, 9(51)\penalty0 (1):\penalty0
  3--24, 1941.

\bibitem[Geman(1982)]{Geman1982a}
Stuart Geman.
\newblock {Almost sure stable oscillations in a large system for randomly
  coupled equations}.
\newblock \emph{SIAM J. Appl. Math.}, 42\penalty0 (4):\penalty0 695--703, 1982.

\bibitem[Geman and Hwang(1982)]{Geman1982}
Stuart Geman and Chii-Ruey Hwang.
\newblock A chaos hypothesis for some large systems of random equations.
\newblock \emph{Zeitschrift f{\"u}r Wahrscheinlichkeitstheorie und Verwandte
  Gebiete}, 60\penalty0 (3):\penalty0 291--314, 1982.

\bibitem[Gers et~al.(1999)Gers, Schmidhuber, and Cummins]{Gers1999}
Felix~A. Gers, J.~Schmidhuber, and Fred Cummins.
\newblock {Learning to Forget: Continual Prediction with LSTM}.
\newblock In \emph{International Conference on Artificial Neural Networks},
  volume~2, pages 850--855, 1999.

\bibitem[Gilboa et~al.(2019)Gilboa, Chang, Chen, Yang, Schoenholz, Chi, and
  Pennington]{Gilboa2019}
Dar Gilboa, Bo~Chang, Minmin Chen, Greg Yang, Samuel~S Schoenholz, Ed~H Chi,
  and Jeffrey Pennington.
\newblock {Dynamical Isometry and a Mean Field Theory of LSTMs and GRUs}.
\newblock \emph{arXiv preprint arXiv:1901.08987}, 2019.

\bibitem[Glorot and Bengio(2010)]{glorot2010understanding}
Xavier Glorot and Yoshua Bengio.
\newblock Understanding the difficulty of training deep feedforward neural
  networks.
\newblock In \emph{Proceedings of the thirteenth international conference on
  artificial intelligence and statistics}, pages 249--256, 2010.

\bibitem[Graves(2013)]{graves2013}
Alex Graves.
\newblock Generating sequences with recurrent neural networks.
\newblock \emph{arXiv preprint arXiv:1308.0850}, 2013.

\bibitem[Hochreiter and Schmidhuber(1997)]{Hochreiter1997}
Sepp Hochreiter and J{\"u}rgen Schmidhuber.
\newblock Long short-term memory.
\newblock \emph{Neural computation}, 9\penalty0 (8):\penalty0 1735--1780, 1997.

\bibitem[Hochreiter et~al.(2001)Hochreiter, Bengio, Frasconi, and
  Schmidhuber]{hochreiter2001field}
Sepp Hochreiter, Yoshua Bengio, Paolo Frasconi, and J{\"u}rgen Schmidhuber.
\newblock Gradient flow in recurrent nets: The difficulty of learning long-term
  dependencies.
\newblock In \emph{A Field Guide to Dynamical Recurrent Neural Networks}, pages
  237--243. Wiley-IEEE Press, 2001.

\bibitem[Jacot et~al.(2018)Jacot, Gabriel, and Hongler]{jacot2018}
Arthur Jacot, Franck Gabriel, and Clement Hongler.
\newblock Neural tangent kernel: Convergence and generalization in neural
  networks.
\newblock In \emph{Advances in Neural Information Processing Systems}, pages
  8571--8580. 2018.

\bibitem[Jing et~al.(2019)Jing, Gulcehre, Peurifoy, Shen, Tegmark, Soljacic,
  and Bengio]{jing2019gated}
Li~Jing, Caglar Gulcehre, John Peurifoy, Yichen Shen, Max Tegmark, Marin
  Soljacic, and Yoshua Bengio.
\newblock Gated orthogonal recurrent units: On learning to forget.
\newblock \emph{Neural computation}, 31\penalty0 (4):\penalty0 765--783, 2019.

\bibitem[Jordan et~al.(2019)Jordan, Sokol, and Park]{Park2019arxiv}
Ian~D Jordan, Piotr~Aleksander Sokol, and Il~Memming Park.
\newblock Gated recurrent units viewed through the lens of continuous time
  dynamical systems.
\newblock \emph{arXiv preprint arXiv:1906.01005}, 2019.

\bibitem[Kanai et~al.(2017)Kanai, Fujiwara, and Iwamura]{Kanai2017}
Sekitoshi Kanai, Yasuhiro Fujiwara, and Sotetsu Iwamura.
\newblock Preventing gradient explosions in gated recurrent units.
\newblock In \emph{Advances in Neural Information Processing Systems}, pages
  435--444, 2017.

\bibitem[Kerg et~al.(2019)Kerg, Goyette, Puelma~Touzel, Gidel, Vorontsov,
  Bengio, and Lajoie]{kerg2019non}
Giancarlo Kerg, Kyle Goyette, Maximilian Puelma~Touzel, Gauthier Gidel, Eugene
  Vorontsov, Yoshua Bengio, and Guillaume Lajoie.
\newblock Non-normal recurrent neural network (nnrnn): learning long time
  dependencies while improving expressivity with transient dynamics.
\newblock In \emph{Advances in Neural Information Processing Systems 32}, pages
  13613--13623. 2019.

\bibitem[Kiros et~al.(2015)Kiros, Zhu, Salakhutdinov, Zemel, Urtasun, Torralba,
  and Fidler]{Kiros2015}
Ryan Kiros, Yukun Zhu, Ruslan~R Salakhutdinov, Richard Zemel, Raquel Urtasun,
  Antonio Torralba, and Sanja Fidler.
\newblock Skip-thought vectors.
\newblock In \emph{Advances in Neural Information Processing Systems}, pages
  3294--3302, 2015.

\bibitem[Lee et~al.(2018)Lee, Bahri, Novak, Schoenholz, Pennington, and
  Sohl-dickstein]{Lee2018}
Jaehoon Lee, Yasaman Bahri, Roman Novak, Samuel~S Schoenholz, Jeffrey
  Pennington, and Jascha Sohl-dickstein.
\newblock {Deep Neural Networks as Gaussian Processes}.
\newblock In \emph{International Conference on Learning Representations
  (ICLR)}, 2018.

\bibitem[Maheswaranathan et~al.(2019)Maheswaranathan, Williams, Golub, Ganguli,
  and Sussillo]{Niru2019}
Niru Maheswaranathan, Alex~H Williams, Matthew~D Golub, Surya Ganguli, and
  David Sussillo.
\newblock {Line attractor dynamics in recurrent networks for sentiment
  classification}.
\newblock In \emph{International Conference on Machine Learning (ICML)}, 2019.

\bibitem[Marti et~al.(2018)Marti, Brunel, and Ostojic]{Marti2017}
Daniel Marti, Nicolas Brunel, and Srdjan Ostojic.
\newblock {Correlations between synapses in pairs of neurons slow down dynamics
  in randomly connected neural networks}.
\newblock \emph{Physical Review E}, 97:\penalty0 062314, 2018.

\bibitem[Mehlig and Chalker(2000)]{mehlig2000statistical}
Bernhard Mehlig and John~T Chalker.
\newblock Statistical properties of eigenvectors in non-hermitian gaussian
  random matrix ensembles.
\newblock \emph{Journal of Mathematical Physics}, 41\penalty0 (5):\penalty0
  3233--3256, 2000.

\bibitem[Pascanu et~al.(2013)Pascanu, Mikolov, and Bengio]{Pascanu2013}
Razvan Pascanu, Tomas Mikolov, and Yoshua Bengio.
\newblock On the difficulty of training recurrent neural networks.
\newblock In \emph{International Conference on Machine Learning (ICML)}, pages
  1310--1318, 2013.

\bibitem[Pennington et~al.(2017)Pennington, Schoenholz, and
  Ganguli]{pennington2017resurrecting}
Jeffrey Pennington, Samuel Schoenholz, and Surya Ganguli.
\newblock Resurrecting the sigmoid in deep learning through dynamical isometry:
  theory and practice.
\newblock In \emph{Advances in Neural Information Processing Systems}, pages
  4785--4795. 2017.

\bibitem[Schoenholz et~al.(2017)Schoenholz, Gilmer, Ganguli, and
  Sohl-Dickstein]{schoenholz2016deep}
Samuel~S Schoenholz, Justin Gilmer, Surya Ganguli, and Jascha Sohl-Dickstein.
\newblock Deep information propagation.
\newblock In \emph{International Conference on Learning Representations
  (ICLR)}, 2017.

\bibitem[Sompolinsky et~al.(1988)Sompolinsky, Crisanti, and
  Sommers]{Sompolinsky1988a}
H.~Sompolinsky, A.~Crisanti, and H.~J. Sommers.
\newblock {Chaos in random neural networks}.
\newblock \emph{Physical Review Letters}, 61\penalty0 (3):\penalty0 259, 1988.

\bibitem[Stern et~al.(2014)Stern, Sompolinsky, and Abbott]{Stern2014}
M~Stern, H~Sompolinsky, and L~F Abbott.
\newblock {Dynamics of random neural networks with bistable units}.
\newblock \emph{Physical Review E}, 90:\penalty0 062710, 2014.

\bibitem[Sutskever et~al.(2012)Sutskever, Martens, Dahl, and
  Hinton]{Sutskever2012}
Ilya Sutskever, James Martens, George Dahl, and Geoffrey Hinton.
\newblock {On the importance of initialization and momentum in deep learning}.
\newblock In \emph{International Conference on Machine Learning (ICML)}, 2012.

\bibitem[Sutskever et~al.(2014)Sutskever, Vinyals, and Le]{Sutskever2014}
Ilya Sutskever, Oriol Vinyals, and Quoc~V Le.
\newblock Sequence to sequence learning with neural networks.
\newblock In \emph{Advances in Neural Information Processing Systems 27}, pages
  3104--3112. 2014.

\bibitem[Tallec and Ollivier(2018)]{tallec2018chrono}
Corentin Tallec and Yann Ollivier.
\newblock Can recurrent neural networks warp time?
\newblock In \emph{International Conference on Learning Representations
  (ICLR)}, 2018.

\bibitem[Wainrib and Touboul(2013)]{wainrib2013}
Gilles Wainrib and Jonathan Touboul.
\newblock Topological and dynamical complexity of random neural networks.
\newblock \emph{Physical review letters}, 110\penalty0 (11):\penalty0 118101,
  2013.

\end{thebibliography}

\appendix


\section{Mean Field Theory for GRUs and LSTMs}\label{app:gru-dmft}

In this appendix, we describe the dynamical mean field equations for GRUs, and elaborate on aspects relevant for the analysis in this paper. A more thorough study of the dynamical equations which follow from the MFT will be undertaken in future work.

The mean field theory for the GRU uses the central limit theorem and the chaos hypothesis (\cite{Amari1972,Geman1982,Cessac1995}) to replace, for a given neuron, its $N$ interactions with an effective stochastic variable. Concretely, the central limit theorem is invoked to write
\begin{align}
\left(U_{z} {\bf h}_{t} + {\bf b}_{z}\right)_{i} \to \zeta_{t}\\
\left(U_{r} {\bf h}_{t} + {\bf b}_{r}\right)_{i} \sim \xi_{t} \\
\left(U_{h} ( {\bf h}_{t} \odot {\bf r}_{t}) + {\bf b}_{h}\right)_{i} \to y_{t}
\end{align}

where $(\zeta_{t}, \xi_{t}, y_{t})$ are independent Gaussian processes characterized by their moments 
\begin{align}
\mathbbm{E}[y_{t}] &= b_{h}, \quad {\rm cov}[ y_{t} y_{t'} ] = a_{h}^{2} C_{r}(t,t') C_{h}(t,t') + v_{h}, \label{eq:y-mft}\\
\mathbbm{E}[\zeta_{t}] &= b_{z}, \quad {\rm cov}[ \zeta_{t} \zeta_{t}] = a_{z}^{2} C_{h}(t, t') + v_{z}, \label{eq:zeta-mft}\\
 \mathbbm{E}[\xi_{t}] &= b_{r}, \quad {\rm cov}[ \xi_{t} \xi_{t'}] = a_{r}^{2} C_{h}(t, t') + v_{r}, \label{eq:r-mft}
\end{align}

with the kernels related to the correlation functions
\begin{align}
C_{h}(t,t') &= \mathbbm{E}[ h_{t} h_{t'}], \quad C_{r}(t, t') = \mathbbm{E}[ \sigma(\xi_{t}) \sigma(\xi_{t'})].
\end{align}

A crucial aspect of the dynamical mean-field theory (DMFT) is that the Gaussian kernels are non-Markovian (i.e. noise at different time steps are correlated), and must be determined self-consistently from the correlation function of the hidden state variable $h_{t}$. Thus, as the network size $N \to \infty$, each neuron becomes essentially independent of the other neurons. The $N$ dimensional dynamical system is then reduced to a one-dimensional stochastic difference equation for the hidden state variable, where the noise term 
embodies the effective interaction with the rest of the network
\begin{align}
h_{t} &= \sigma(\zeta_{t-1})h_{t-1} + (1 - \sigma(\zeta_{t-1})) \phi(y_{t-1}).
\end{align}
In what follows, we will also make use of the additional correlation functions, which may be viewed as transforms of a Gaussian process
\begin{align}
C_{z}(t, t') = \mathbbm{E}[ \sigma(\zeta_{t}) \sigma(\zeta_{t'})], \quad C_{\phi}(t, t') = \mathbbm{E}[ \phi(y_{t}) \phi(y_{t'})]	, \quad \kappa_{t} = \mathbbm{E}[ \sigma(\zeta_{t}) ].
\end{align}
For long times, the autonomous dynamics will approach a steady state described either by a fixed point or chaotic attractor. Limit cycles require fine tuned weight matrices, and are expected to have vanishing probability in the limit of large $N$. Furthermore, upon reaching this limiting distribution, the correlation function should only depend on the absolute difference in times, e.g. $C_{h}(t, t+ k) \to C_{h}(k)$ and
averages become time-independent e.g. $\kappa_t = \kappa$. The DMFT in the steady state reduces to the second order difference equation
\begin{align}
	&C_{h}(k) - \kappa C_{h}(k-1) - \kappa C_{h}(k+1)  + C_{z}(k) C_{h}(k) = \left(1 - 2 \kappa 	+ C_{z}(k)\right) C_{\phi}(k),\label{eq:dmft1}\\
	&C_{h}(0) + C_{z}(0) C_{h}(0) - 2 \kappa C_{h}(1) = (1 - 2 \kappa + C_{z}(0)) C_{\phi}(0).\label{eq:dmft2}
\end{align}
Fixed points are described by time-independent distributions. Assuming time-independent in (\ref{eq:dmft1}) and (\ref{eq:dmft2}) gives the implicit equation (\ref{eq:FP1}) stated in the main text.

Another useful relation follows from the fact that the autocorrelation function is bounded at all times $C_{h}(|t-t'|) \le C_{h}(0)$. Using this in (\ref{eq:dmft2}) leads to the bound for the equal-time correlation function for time-dependent solutions

\begin{align}
C_{h}(0)	\le  C_{\phi}(0).
\end{align}
If we let $C_{h}^{FP}$ denote the fixed point solution $C_{h}^{FP} = C_{\phi}(0)$, then this inequality implies $C_{h}(0) \le C_{h}^{FP}$. Therefore, the fixed point (time-independent) solution is an upper bound on the equal-time autocorrelation function in a time-dependent steady state.

\paragraph{LSTM} Applying a similar reasoning as above, we arrive at the following effective stochastic difference equation for the LSTM (\ref{eq:lstm_eom})
\begin{align}
c_{t} &= \sigma(\eta_{t-1}^{ f}) c_{t-1} + \sigma(\eta_{t-1}^{i}) \phi(y_{t-1}),\\
h_{t} & = \sigma(\eta_{t-1}^{o}) \phi(c_{t}),
\end{align}
where $y_{t}$ is a Gaussian process with mean $b_{h}$ and covariance
\begin{align}
    {\rm cov}\left[ y_{t} y_{t'}\right] = a_{h}^{2} C_{h}(t,t') + v_{h},
\end{align}
and $\eta_{t}^{k}$ for $k \in \{ f, i, o\}$ are Gaussian processes with mean $b_{k}$ and covariance

\begin{align}
{\rm cov}\left[ \eta_{t}^{k} \eta_{t'}^{k'} \right] = 	\delta_{kk'} \left(a_{k}^{2} C_{h}(t, t') + v_{k}\right), \quad C_{h}(t, t') = \mathbbm{E}\left[ h_{t} h_{t'}\right].
\end{align}
$c_{t}$ and $h_{t}$ governed by such a stochastic difference equation will not be gaussian, and this makes the fixed point analysis difficult. However, we see at the very least that the $\eta_{t}^{k}$ are independent. Henceforth, for simplicity we denote $k_{t} = \sigma(\eta_{t-1}^{k})$ and $k_{t}' =\sigma'(\eta_{t-1}^{k})$. Similarly for the GRU  $z_{t} = \sigma(\zeta_{t-1})$, $z'_{t} =  \sigma'(\zeta_{t-1})$,
    $r_{t} = \sigma(\xi_{t-1})$, and $r_{t}' = \sigma'(\xi_{t-1})$


\section{GRUs: proof of Theorem \ref{thm:GRU-spec-curve}}\label{app:gru-rmt}

Here we outline the proof of our main result concerning the spectral curve. Our approach uses the ideas about hermitization method and free probability theory to truncate the cumulants.
We are interested in determining the spectral properties of the state-to-state Jacobian (\ref{eq:jac}). In the main text, we suggested that it is advantageous to look at an enlarged Jacobian ${\bf M}_{t}$ defined in eq.(\ref{eq:jac-M}). This extension is properly thought of as an implementation of the linearization method (see e.g. \cite{Belinschi2018} and references therein) to deal with products of random matrices. The linearization in this problem happens to be naturally suggested by the structure of the GRU network dynamics. 

Consider the generalized resolvent 
\begin{align}
{\bf G}(\lambda) = \left( {\bf I}_{\lambda} - {\bf M}_{t}\right)^{-1} = \left( \begin{array}{ccc}
 {\bf G}_{11} & {\bf G}_{12} & {\bf G}_{13}\\
{\bf G}_{21} & {\bf G}_{22} & {\bf G}_{23}\\
{\bf G}_{31}&{\bf  G}_{32} & {\bf G}_{33}	
 \end{array}\right).
\end{align}

It is simple to check that ${\bf G}_{11} = \left( \lambda  - {\bf J}_{t}\right)^{-1}$ is the resolvent of Jacobian. It is convenient to decompose the extended Jacobian (\ref{eq:jac-M}) in the form (dropping the time index for simplicity) ${\bf M} = \hat{A} + \hat{B} \hat{U} \hat{C}$ where 
\begin{align*}
\hat{A} &= \left( \begin{array}{ccc}
 {\bf \hat z}  & 0 & \hat{\bf h} - \hat{\phi}({\bf y}) \\
 0 & 0 & 0\\
 0 & 0 & 0
 \end{array}\right)	, \quad \hat{B} = \left( \begin{array}{ccc}
  (\mathbbm{1} - {\bf \hat z})\hat{\phi}'({\bf y})  & 0 & 0\\
 0 &  {\bf r}' & 0\\
0 & 0 &  {\bf z}'
 \end{array}\right), \quad \hat{C} &= \left( \begin{array}{ccc}
{\bf r} &  {\bf h} & 0\\
 \mathbbm{1} & 0 & 0\\
\mathbbm{1} & 0 & 0
 \end{array}\right),
\end{align*}
and $\hat{U} = {\rm bdiag}(U_{h},U_{r},U_{z})$. Following the well-known method of hermitian reduction, we introduce an auxiliary complex variable $\eta$ and consider the resolvent $\mathcal{G}$, whose inverse is
\begin{align}
\mathcal{G}^{-1} = \left( \begin{array}{cc}
 \eta \mathbbm{1}_{3N} & {\bf I}_{\lambda} - {\bf M}	\\
 {\bf I}_{\bar{\lambda}} - {\bf M}^{T} & \eta \mathbbm{1}_{3N}
 \end{array}\right)	 = \left( \begin{array}{cc}
 \eta \mathbbm{1}_{3N} & {\bf I}_{\lambda} - \hat{A}	\\
 {\bf I}_{\bar{\lambda}} - \hat{A}^{T} & \eta \mathbbm{1}_{3N}
 \end{array}\right) - \left( \begin{array}{cc}
0 &  \hat{B} \hat{U} \hat{C}	\\
 \hat{C}^{T} \hat{U}^{T} \hat{B}^{T} & 0
 \end{array}\right). \label{eq:herm-resolvent}
\end{align}

Denote the four $3N\times 3N$ blocks of $\mathcal{G}$ by $ \mathcal{G}_{ab}$ for $a, b \in \{1,2\}$. We identify $\mathcal{G}_{21} = {\bf G}(\lambda)$ as the generalized resolvent of interest. The expression (\ref{eq:herm-resolvent}) is of the form $\mathcal{G}^{-1} = \mathcal{G}_{0}^{-1} - \mathcal{H}$, where $\mathcal{H}$ contains the Gaussian random variables $\hat{U}$. Since the elements of the matrix $\mathcal{H}$ are gaussian, we utilize the so-called self-consistent Born approximation to compute the self-energy (see e.g. \cite{mehlig2000statistical}). In the language of Free Probability theory, this is equivalent to calculating the $R$ transform by keeping only the second cumulant. This approximation is expected to be exact in the $N \to \infty$ limit. In this limit, the self-energy $\boldsymbol{\Sigma}$ will be block diagonal, with the $N\times N$ blocks given by

\begin{align}
\boldsymbol{\Sigma}_{11}[\mathcal{G}] &= \hat{B} \mathcal{Q}\left[ \hat{C} \mathcal{G}_{22} \hat{C}^{T}\right] \hat{B}^{T}, \quad \boldsymbol{\Sigma}_{22}[\mathcal{G}] = \hat{C}^{T} \mathcal{Q}\left[ \hat{B}^{T} \mathcal{G}_{11} \hat{B}\right] \hat{C},
\end{align}

where the superoperator, defined by the action

\begin{align}
	\mathcal{Q}[ R] = \left( \begin{array}{ccc}
 \frac{a_{h}^{2}}{N} {\rm tr} 	R_{11} & 0 & 0\\
 0 & \frac{a_{r}^{2}}{N} {\rm tr} R_{22} & 0 \\
 0 & 0 & \frac{a_{z}^{2}}{N} {\rm tr} R_{33}
 \end{array}\right),
\end{align}
arises after ensemble averaging over the random connectivity matrices in the block diagonal $\hat{U}$. We have therefore replaced the explicit random variable $\mathcal{H}$ with an effective ``self-energy" and rewritten the expression for the resolvent $\mathcal{G}$ in the form $\mathcal{G}^{-1} = \mathcal{G}_{0}^{-1} - \boldsymbol{\Sigma}[\mathcal{G}]$, known as the Dyson equation. Explicitly, to compute the resolvent we must solve the non-linear system of equations implied by the expression
\begin{align}
\mathcal{G}^{-1} = 	\left( \begin{array}{cc}
- \boldsymbol{\Sigma}_{11}[ \mathcal{G}] & {\bf I}_{\lambda} - \hat{A}	\\
 {\bf I}_{\bar{\lambda}} - \hat{A}^{T} & -\boldsymbol{\Sigma}_{22}[ \mathcal{G}]
 \end{array}\right).
\end{align}
Since the self-energy involves traces of transformed blocks of the resolvent, it is useful to work directly with these transformed matrices. So we define
\begin{align}
\mathcal{F}_{11} = \hat{B}^{T} \mathcal{G}_{11} \hat{B}, \quad \mathcal{F}_{22} = \hat{C} \mathcal{G}_{22} \hat{C}^{T}	.
\end{align}
We will end up needing only the diagonal blocks of these matrices, since this is all that enters the self-energy. For $\mathcal{F}_{11}$, we denote the diagonal block matrices $F_{ii}$, for $i \in \{1, 2, 3\}$, and similarly for $\mathcal{F}_{22}$ with indices $i \in \{4, 5, 6\}$. Finally, we denote the normalized trace $f_{ii} = \frac{1}{N} {\rm tr} F_{ii}$. Furthermore, for simplicity, define the functions $Q = \frac{f_{22} + f_{33}}{f_{11}}$, $X = f_{11} f_{44}$ and $Y= f_{11} f_{55}$. The Dyson equation then leads to a system of equations for various Green's functions which we must solve self-consistently.

To present our results, we remove the bold-face and carat (e.g. ${\bf \hat h} \to h$), and denote 

\begin{align}
\frac{1}{N} {\rm tr} \left( {\bf Q}({\bf h}, {\bf z}, {\bf y}, {\bf r})\right)	 = \mathbbm{E}\left[ Q( h, z, y, r)\right],
\end{align}
where $Q$ is a matrix-valued function of the state variables, e.g. ${\bf Q} = {\bf \hat h}$. In the large $N$ limit where the mean field theory is expected to hold, we may take this to mean that the sum over all neurons is literally replaced by an expectation value over a single neuron with fluctuating fields.

An important object, which is a non-Hermitian generalization of the Stieltjes transform, is the trace of the resolvent $G(\lambda) = \mathbbm{E} \left[  \frac{1}{N}{\rm tr} (\lambda - {\bf J}_{t})^{-1}\right]$. From the Dyson equation, we find
\begin{align}
G(\lambda) = \mathbbm{E}\left[\frac{(\bar{\lambda} - z) (1 - Y h^{2} r'^{2})}{D}	\right],
\end{align}	
where the denominator is given by
\begin{align}
D &=  |\lambda - z|^{2} (1 - Y h^{2} (a_{r}r')^{2} )-  (1 - z)^{2} (a_{h}\phi'(y))^{2} ( Q X + X r^{2}) + X Y Q (1 - z)^{2} (a_{h}\phi'(y))^{2} h^{2} (a_{r}r')^{2} \nonumber\\
&- Q Y (h - \phi(y))^{2} (a_{z}z')^{2} - Y ( r^{2} - Q Y h^{2} (a_{r}r')^{2})(h - \phi(y))^{2} (a_{z}z')^{2},	
\end{align}

and the auxiliary functions must be computed self-consistently to satisfy the system of equations

\begin{align}
1 & = \mathbbm{E}\left[ \frac{(1 - z)^{2} (a_{h}\phi'(y))^{2}r^{2}}{D}\right] +   Q\,  \mathbbm{E}\left[ \frac{(1 - z)^{2} (a_{h}\phi'(y))^{2} \left( 1 - Y h^{2} (a_{r}r')^{2}\right)}{D}\right] \label{eq:gf1}\\
X & = X \mathbbm{E}\left[  \frac{(1 - z)^{2} (a_{h}\phi'(y))^{2} r^{2}}{D}\right] - Q X Y \mathbbm{E}\left[  \frac{(1 - z)^{2} (a_{h}\phi'(y))^{2} h^{2} (a_{r}r')^{2}}{D}\right] + Y \mathbbm{E}\left[  \frac{ r^{2} (h- \phi(y))^{2} (a_{z}z')^{2}}{D}\right] \nonumber \\
&+ Y \mathbbm{E}\left[  \frac{h^{2} (a_{r}r')^{2} |\lambda - z|^{2}}{D}\right] - QY^2 \mathbbm{E}\left[  \frac{h^{2} (a_{r}r')^{2} (h-\phi(y))^{2} (a_{z}z')^{2}}{D}\right] \label{eq:gf2}\\
Y & = X \mathbbm{E}\left[  \frac{(1 - z)^{2} (a_{h}\phi'(y))^{2} }{D}\right] - X Y  \mathbbm{E}\left[  \frac{(1-z)^{2} (a_{h}\phi'(y))^{2} h^{2} (a_{r}r')^{2}}{D}\right] \nonumber\\
& + Y \mathbbm{E}\left[ \frac{(h-\phi(y))^{2} (a_{z}z')^{2}}{D}\right] - Y^{2} \mathbbm{E}\left[  \frac{ h^{2} (a_{r}r')^{2} (h-\phi(y))^{2} (a_{z}z')^{2}}{D}\right] \label{eq:gf3}\\
Q = &Q \Big\{ \mathbbm{E}\left[ \frac{(h-\phi(y))^{2} (a_{z}z')^{2}}{D}\right] - 2 Y \mathbbm{E}\left[ \frac{h^{2} (a_{r}r')^{2} (h-\phi(y))^{2} (a_{z}z')^{2}}{D}\right]\nonumber	\\
& - X \mathbbm{E}\left[ \frac{(1-z)^{2}(a_{h}\phi'(y))^{2} h^{2} (a_{r}r')^{2}}{D}\right] \Big\}  +  \left[\mathbbm{E}\left[ \frac{|\lambda - z|^{2} h^{2} (a_{r}r')^{2}}{D}\right]+ \mathbbm{E}\left[ \frac{r^{2} (h- \phi(y))^{2} (a_{z}z')^{2}}{D}\right]\right] \label{eq:gf4}
\end{align}

Near the boundary of the eigenvalue support, $X$ and $Y$ tend to zero, becoming exactly zero on the boundary. However, $Q$ remains finite. We can use (\ref{eq:gf1}) to eliminate $Q$ from (\ref{eq:gf4}), and after setting $X = Y = 0$, we obtain the equation for the spectral boundary curve

\begin{align}
1 = &\mathbbm{E} \left[ h^{2} (a_{r} r')^{2}\right] \mathbbm{E}\left[ \frac{(a_{h}\phi'( y))^{2} \left( 1 - z\right)^{2}	}{| \lambda - z|^{2}} \right]+ \mathbbm{E}\left[ \frac{(a_{h}\phi'(y))^{2}  r^{2} \left( 1 - z\right)^{2}	}{| \lambda - z|^{2}} \right]\\
&  +\mathbbm{E} \left[ \frac{(h - \phi(y))^{2} (a_{z} z')^{2}	}{| \lambda -z|^{2}} \right] .
\end{align}

After invoking the MFT to factorize certain terms, this simplifies to read 
\begin{align}
& 1 =    \mathbbm{E}[(r^{2} + a_{r}^{2}h^{2} r'^{2})] \mathbbm{E}\left[ a_{h}^{2}(\phi'(y))^{2}\right]  \mathbbm{E}\left[ \frac{(1 - z)^{2}}{|\lambda - z|^{2}}\right] +\mathbbm{E}\left[ \frac{a_{z}^{2}(h-\phi(y))^{2} z'^{2}}{|\lambda - z|^{2}}\right]  .
\end{align}

The spectral boundary will consist of all $\lambda\in\mathbbm{C}$ which satisfy this equation.

In the main text, we introduced a boundary curve function 

\begin{align}
\mathcal{S}(\lambda) = \rho^{2}\mathbbm{E}\left[ \frac{\left( 1 - z\right)^{2}	}{| \lambda - z|^{2}} \right] + a_{z}^{2}\mathbbm{E}\left[ \frac{(h - \phi( y))^{2} (z')^{2}	}{| \lambda - z|^{2}} \right] -1,
\end{align}

which vanishes on the boundary. The difficulty here lies in making sense of the second term. First of all, if the steady state is a fixed point, one must have $h = \phi(y)$, and the second term simply vanishes. However, in general the attractor is not a simple fixed point (since we have shown that the fixed points are unstable), and we cannot expect this term to vanish. The main difficulty with this correlation function is that in general $h$ and $z$ are {\it not} independent, even in the mean field theory, and they may have non-trivial correlations. 

Let us define 
\begin{align}
\mathcal{S}_{\rho, \nu}(\lambda) = 	\rho^{2} \mathbbm{E}\left[ \frac{\left( 1 - z\right)^{2}	}{| \lambda - z|^{2}} \right]+ \nu^{2} \, \mathbbm{E}\left[ \frac{z'^{2}	}{| \lambda - z|^{2}} \right], \quad \nu^{2} = \mathbbm{E}\left[ (h - \phi(y))^{2}\right]\label{eq:S-curve-approx} ,
\end{align}

along with the set

\begin{align}
\Sigma(\rho, \nu) := \{ \lambda \in \mathbbm{C} \, | \, \mathcal{S}_{\rho, \nu}(\lambda) \ge  0\}	.
\end{align}

By Cauchy-Schwarz, we have that $\mathcal{S}(\lambda) \le \mathcal{S}_{\rho, \nu}(\lambda)$, and by the positivity of the second term on the RHS of (\ref{eq:spec-curve}), it follows that $\mathcal{S}_{\rho, 0}(\lambda) \le \mathcal{S}(\lambda)$. Therefore, we conclude that 

\begin{align}
\Sigma(\rho, 0) \subseteq \Sigma \subseteq \Sigma(\rho, \nu),	
\end{align}

where $\Sigma$ is the spectral support defined in Thm.(\ref{thm:GRU-spec-curve}). Thus we obtain bounding curves for the spectral density described parametrically by $\mathcal{S}(\rho, \nu) = 0$ and $\mathcal{S}(\rho, 0) = 0$, both of which are expressible in terms of simple correlation functions.

Obtaining the trace of the resolvent is tractable in limited cases.
The Dyson equation for the Green's function proves to be intractable when both reset and update gates have nonzero variance. However, when the reset gate variance $a_{r} = 0$, the equations simplify considerably and we can determine the trace of the resolvent:

\begin{proposition}[Trace of the Resolvent]\label{prop:gru-resolvent}
For $a_{r} =0$, the trace of the resolvent for the GRU instantaneous Jacobian is
\begin{align}
G(\lambda) = \begin{dcases}
 	\mathbbm{E}\left[ \frac{\bar{\lambda} - z_{t}}{|\lambda - z_{t}|^{2} + F(\lambda) s_{t}} \right], \quad &\lambda \in \Sigma \label{eq:zgateG}\\
 	\mathbbm{E}\left[ \frac{1}{\lambda - z_{t}}\right], \quad & \lambda \in \Sigma^{c}
 \end{dcases}	
\end{align}

where $s_{t} = r^{2} a_{h}^{2} (1 - z_{t})^{2}  \phi'(y_t)^{2} + a_{z}^{2} z_{t}'^{2} ( h_{t-1} - \phi(y_{t}))^{2}$. For $\lambda \in \Sigma$, $F(\lambda)\ge 0$ is given implicitly by the equation
\begin{align}
1 = \mathbbm{E}\left[ \frac{s_{t}}{|\lambda - z_{t}|^{2} + F(\lambda) s_{t}}\right]\label{eq:zgateF}	, 
\end{align}
while for $\lambda \in \Sigma^{c}$, $F = 0$. 

\end{proposition}

The proof follows easily after specializing Eqs.(\ref{eq:gf1} - \ref{eq:gf4}) for $a_{r} = 0$ in which the reset gate is fixed $r_{t} = r = \sigma(b_{r})$. It is apparent by inspection of Eqs. (\ref{eq:gf2}) and (\ref{eq:gf3}) that $X = r^{2} Y$. Consequently, by introducing the function $F(\lambda) = - (Q Y + Y r^{2})$, we arrive at the statement of the proposition.


\section{GRU: accumulation of eigenvalues at $\lambda=1 $}\label{app:only-update}

We provide the details of analytical demonstration of accumulation shown visually in Fig.(\ref{fig:gru-spectrum} f) as a function of update gate variance $a_{z}$ (with reset gate inactive: $a_{r} = 0$), and provide a derivation of the $a_{z}\to \infty$ limiting eigenvalue density Eq. \ref{eq:den-update-only} presented in the main text. Our starting point is the trace of the resolvent given by Prop. (\ref{prop:gru-resolvent}).

For $F(\lambda) \ge 0$, we must solve (\ref{eq:zgateF}) first, then insert this solution back into (\ref{eq:zgateG}) to determine the trace of the resolvent. Now we make two assumptions: first, we use  a piecewise linear approximation for $\phi(y)$ (the ``hard-tanh")

\begin{align}
\phi(x) &= \begin{cases}
 1 , \quad 	x > 1\\
 x, \quad - 1 < x < 1\\
 - 1 , \quad x < - 1
 \end{cases} , \quad \phi'(x)  = \begin{cases}
 0 , \quad 	x > 1\\
1 , \quad - 1 < x < 1\\
 0 , \quad x < - 1
 \end{cases}, \label{eq:piecewise}
\end{align}

and secondly we assume $P(\zeta, y) = P(\zeta) P(y)$ which is justified
in the mean-field limit. Evaluate the Gaussian integral over $y$ using (\ref{eq:y-mft}) gives

\begin{align}
G(\lambda) &= \eta \mathbbm{E}\left[ \frac{\bar{\lambda} - z}{ |\lambda - z|^{2} + F (1 - z)^{2} r^{2} a_{h}^{2}} \right]_{\zeta} + (1 - \eta) \mathbbm{E}\left[ \frac{1}{\lambda - z}\right]\label{eq:G_approx}\\
1 & = 	\eta \mathbbm{E} \left[ \frac{(1 - z)^{2} r^{2} a_{h}^{2}}{|\lambda - z|^{2} + F (1 - z)^{2} r^{2} a_{h}^{2}} \right]_{\zeta}\label{eq:F_approx}
\end{align}

where
\begin{align}
    \eta =  \int_{-1}^{1} \frac{1}{\sqrt{2 \pi a_{h}^{2} r^{2} C_{h}}}\exp\left( - \frac{y^{2}}{2 a_{h}^{2} r^{2} C_{h}} \right)  
\end{align}
can be interpretted as the fraction of unsaturated activations $\eta \approx \mathbbm{E}\left[ (\phi')^{2}\right]$ (which is exact for the hard-tanh). We are left with a Gaussian integral over $\zeta$, which has the distribution (\ref{eq:zeta-mft})

\begin{align}
    P(\zeta) = \frac{1}{\sqrt{2 a_{z}^{2} C_{h}}} \exp\left( - \frac{(\zeta - b_{z})^{2}}{2 a_{z}^{2} C_{h}}\right) .
\end{align}

We consider first the limit $a_{z} = \infty$ to derive Eq.(\ref{eq:den-update-only}), then consider the large $a_{z}$ limit to obtain a scaling function describing the cumulative distribution function near $\lambda = 1$. 

\paragraph{Resolvent for Binary Update Gate} For $a_{z}, b_{z} \to \infty$ with $b_{z}/a_{z} = \beta$ kept fixed, we approximate $z = \sigma(\zeta)$ as a binary random variable with distribution $P(z) = \alpha \delta(z) + (1 - \alpha) \delta(z - 1)$, where 
\begin{align}
    \alpha = 1 - \mathbbm{E}[\sigma(\zeta)] = \int Dx \left( 1 + \exp \left( - a_{z}( \sqrt{C_{h}} x + \beta)\right)\right)^{-1}  \approx \frac{1}{2} {\rm erfc} \left( \frac{\beta}{\sqrt{2 C_{h}}}\right), \label{eq:alpha2}
\end{align}
which is the expression quoted in (\ref{eq:alpha}). The expectation over $\zeta$ in (\ref{eq:zgateG}) and (\ref{eq:zgateF}) is then trivial, and solving for $F$ we find the trace of the resolvent 

\begin{align}
    G(\lambda) = \begin{dcases}
 \frac{ (1 - \alpha) }{\lambda - 1}+ \frac{\alpha(1 - \eta)}{\lambda} + \frac{\bar{\lambda}}{r^{2} a_{h}^{2} } , &\quad |\lambda| \le \sqrt{\eta \alpha} r a_{h}\\
    \frac{(1 - \alpha)}{\lambda - 1} + \frac{\alpha}{\lambda} , &\quad |\lambda|> \sqrt{\eta \alpha} r a_{h}.
    \end{dcases}
\end{align}

Using $\mu(\lambda) = (1/\pi) \partial_{\bar{\lambda}} G(\lambda)$ then gives Eq.(\ref{eq:den-update-only}). It is especially interesting to note the delta functional concentration of eigenvalues at $\lambda = 1$. Now, we seek to understand how the accumulation grows with increasing $a_{z}$, in particular with the goal of obtaining a scaling function to understand the eigenvalue cumulative distribution function (CDF) centered at $\lambda = 1$ displayed in Fig.(\ref{fig:gru-spectrum}).

\paragraph{Scaling Function for CDF}
In order to obtain a scaling function to describe the accumulation of eigenvalues at $\lambda = 1$, we make a series of approximations. The good agreement we find with numerical experiments appears to justify these assumptions. First, we assume that we are at or near a fixed point, and thus can neglect terms involving $h - \phi(y)$ in the expression for the resolvent (this is also justified when $a_{z}$ is large and the update gate is mostly saturated). To get an explicit expression, we use the hard-tanh activation function, so our starting point is again Eqs.(\ref{eq:G_approx}) and (\ref{eq:F_approx}).

We would like to obtain the scaling behavior of the cumulative distribution function $CDF(r) = P(|\lambda - 1|<r)$ close to the $\lambda = 1$. With a little foresight, we parameterize the coordinate $\lambda$ as
\begin{align}
\lambda = 1 + \epsilon, \quad \epsilon = e^{ - \tilde{a}_{z}  c + i \theta}	, \quad \tilde{a}_{z} = a_{z} \sqrt{C_{h}}.
\end{align}
The CDF follows from the resolvent by the contour integral

\begin{align}
CDF(|\epsilon|) = \frac{1}{2\pi i} \oint_{\mathcal{C}} d \epsilon \, G(1 + \epsilon) = \frac{1}{2\pi} \int d\theta \, \epsilon G(1 + \epsilon).
\end{align}

We must therefore find an approximation of $G(1 + \epsilon)$ for small $\epsilon$. The first observation is that $F$ is not singular at $\lambda = 1$ (though we observe that it is not necessarily smooth), since
\begin{align}
1 =  \eta \mathbbm{E}\left[ \frac{(1 - z)^{2} r^{2} a_{h}^{2}}{(1 - z)^{2} + F(1) (1 - z)^{2} r^{2} a_{h}^{2}}\right] =  \frac{\eta r^{2} a_{h}^{2}}{1 + F(1) r^{2} a_{h}^{2}},
\end{align}
so we assume for simplicity $F(1+\epsilon) \approx F(1)$. Using this, we evaluate $G(\lambda)$. We may rewrite the resulting Gaussian integral using the fact that when $b_{z} = 0$, $\sigma(\zeta) = 1 - \sigma(-\zeta)$, and defining $G(\lambda) = G^{a}(\lambda) + G^{b}(\lambda)$, we get

\begin{align}
G^{a}(\lambda) & = \eta \int D  x \frac{ e^{- \tilde{a}_{z} c - i \theta}  + 2e^{- \tilde{a}_{z} (c+x) - i \theta}  + e^{- \tilde{a}_{z} (2 x + c) - i \theta}   + 1 + e^{ - \tilde{a}_{z} x}}{| e^{ - \tilde{a}_{z} c+ i \theta} +  e^{ - \tilde{a}_{z} (c+x)+ i \theta} +1|^{2} + F r^{2} a_{h}^{2}} ,\\
G^{b}(\lambda)& = (1 - \eta ) \int D x \frac{1 + e^{ - \tilde{a}_{z} x}}{\epsilon + 1  + e^{ - \tilde{a}_{z} (x+c) + i \theta}}.
\end{align}
We approximate these integrals for large $a_{z}$ in four regions: I) $x \in (-\infty, -c)$, II) $x \in (-c, -c/2)$, III) $x \in (-c/2,0)$ and IV) $x \in (0,\infty)$. We begin with $G^{a}$. Approximating the integrals by keeping only the dominant terms for large $a_{z}$, we get 

\begin{align}
G_{I}^{a}(\lambda) &\approx  \eta \int_{-\infty}^{-c} Dx \frac{	e^{ - \tilde{a}_{z} c - i \theta} e^{ - 2 \tilde{a}_{z}  x}}{e^{ - 2 \tilde{a}_{z} c - 2 a_{z}  x}}  = \frac{1}{\epsilon} \frac{\eta}{2} {\rm erfc}\left( \frac{c }{\sqrt{2}}\right),\\
G_{II}^{a}(\lambda) &\approx   \eta \frac{\bar{\epsilon}}{|\epsilon + 1|^{2} + F r^{2} a_{h}^{2}} \int_{- c }^{-c/2} Dx e^{ - 2 \tilde{a}_{z}x},\\
G_{III}^{a}(\lambda) &\approx \eta \frac{1}{|\epsilon + 1|^{2} + F r^{2} a_{h}^{2}} 	\left( \int_{-c/2}^{0} Dx e^{ - \tilde{a}_{z}  x} \right),\\
G_{IV}^{a}(\lambda)  &\approx  \frac{\eta}{2} \frac{1}{1 + F r^{2} a_{h}^{2}}	.
\end{align}
Repeating a similar analysis for $G^{b}(\lambda)$, we get the approximate expression
\begin{align}
    G_{I}^{b}(\lambda) &\approx  \frac{1}{\epsilon} \frac{(1-\eta)}{2}{\rm erfc}\left( \frac{c}{\sqrt{2}}\right), \quad G_{II}^{b} + G_{III}^{b} + G_{IV}^{b}  \approx  (1 - \eta) \frac{1}{\epsilon + 1} \int_{-c}^{\infty} Dx e^{ - \tilde{a}_{z} x}
\end{align}

Each of these contributions to the trace of the resolvent will produce a different contribution to the CDF: $G_{II}$,$G_{III}$ and $G_{IV}$ will to leading order give a contribution that vanishes as $a_{z} \to \infty$. In contrast, $G_{I}$ gives a contribution which does not vanish in this limit - in fact, it grows as $a_{z} \to \infty$! Therefore, focusing on this contribution we arrive at the CDF

\begin{align}
P(|\lambda -1|< e^{ - \tilde{a}_{z} c}) = \frac{1	}{2} {\rm erfc} \left( \frac{c}{\sqrt{2 }}\right).
\end{align}

Writing $r = e^{ - a_{z} \sqrt{C_{h}} c}$, and $c = - \log (r)/a_{z} \sqrt{C_{h}}$, we get the cumulative distribution function

\begin{align}
P(|\lambda - 1| < r) \approx \frac{1}{2} {\rm erfc} \left( -\frac{	\log (r)}{a_{z} \sqrt{2 C_{h}}}\right).
\end{align}

This motivates the scaling function we plot in Figs.(\ref{fig:gru-spectrum} f) and (\ref{fig:lstm-plots} f), 

\begin{align}
{\rm CDF}(r) = c_{1} {\rm erfc} \left( c_{2}/a_{z}\right),\label{eq:cdf_scaling}
\end{align}
treating $c_{1}$ and $c_{2}$ as fitting parameters. This function appears to capture well the scaling behavior of the CDF with $a_{z}$ in a small ball around $\lambda = 1$, even at small $a_{z}$ where the derivation does not strictly make sense. This derivation shows that the dominant contribution to the CDF comes, quite naturally, from the domain of $\zeta$ over which the update gate $\sigma(\zeta) \approx  1$.


\subsection{Proof of Theorem (\ref{thm:pinching}) on Spectral Pinching}\label{app:pinching}

Here we present a proof of the spectral pinching announced in Thm.(\ref{thm:pinching}). We start with the equation for the spectral boundary curve for fixed points, i.e. Eq.(\ref{eq:spec-curve}) with the second term set to zero, or equivalently $\mathcal{S}_{\rho,0}$ defined in (\ref{eq:S-curve-approx}). We propose an ansatz for the curve $\lambda(s)$ close to the leading edge

\begin{align}
\lambda  = 1+\lambda_{0} e^{ - c \tilde{a}_{z}}, \quad \tilde{a}_{z} = \sqrt{C_{h}} a_{z}	,\label{eq:ansatz}
\end{align}
where $\lambda_{0}\in \mathbbm{C}$ is a complex constant, and $c\in \mathbbm{R}^{+}$ is an order one real constant. We consider the large $a_{z}, b_{z}$ limit while keeping $C_{h}$ fixed, and the ratio $b_{z}/a_{z} =  \beta $ also fixed. For convenience, let $\tilde{\beta} = \beta/\sqrt{C_{h}}$

We proceed now to constrain $c$ via the support equation $\mathcal{S}_{\rho,0}(\lambda) = 1$. For large $a_{z}$ we may approximate this boundary curve equation

\begin{align}
1 &=  \rho^{2} \int Dx \frac{(1 - \sigma(\tilde{a}_{z} ( x + \tilde{\beta})	))^{2}}{|\lambda - \sigma(\tilde{a}_{z} (x + \tilde{\beta} ))|^{2}} = \rho^{2} \int Dx \frac{1}{|1 + \lambda_{0}e^{ - c \tilde{a}_{z}} (1 + e^{ \tilde{a}_{z} (x + \tilde{\beta})})|^{2}}  \\
& \approx \rho^{2} \left\{ \int_{-\infty}^{c-\tilde{\beta}} Dx \frac{1}{|1 + \lambda_{0} e^{ - c \tilde{a}_{z}}|^{2}} + \frac{1}{|\lambda_{0}|^{2}}\int_{c-\tilde{\beta}}^{\infty} Dx \, e^{ 2 \tilde{a}_{z} (c - \tilde{\beta} - x)}  \right\}\\
& \approx  \rho^{2} \frac{1}{2 | 1 + \lambda_{0} e^{ - c \tilde{a}_{z}}|^{2}} {\rm erfc}\left( - \frac{(c-\tilde{\beta})}{\sqrt{2}}\right) + \frac{\rho^{2}}{2|\lambda_{0}|^{2}} e^{ 2 \tilde{a}_{z}(\tilde{a}_{z} - \tilde{\beta} + c)} {\rm erfc} \left( \frac{2 \tilde{a}_{z} - \tilde{\beta} + c}{\sqrt{2}}\right)
\end{align}

Now sending $a_{z} \to \infty$, we recover the implicit equation for $c$ quoted in (\ref{eq:c}).


\section{Spectral Radius via Gelfand's Formula and Proof of Proposition (\ref{prop:reset-spectral-radius})}\label{app:gru-reset-spectral-radius}

To understand the role of the reset gate, we set $a_{z} = 0$ and fix $b_{z}$. We want to show that the spectral radius will grow with increasing $a_{r}$. This could be calculated using our formula for the spectral curve, but here we provide an alternate approach using Gelfand's formula and take an average over the random weights. The Gelfand formula states that the spectral radius is given by the limit

\begin{align}
	\rho({\bf J}_{t}) = \lim_{n \to \infty} || {\bf J}_{t}^{n}||^{1/n}
\end{align}

for any matrix norm $|| \cdot ||$. We make a self-averaging assumption by taking

\begin{align}
\left\langle \lim_{n \to \infty} || {\bf J}_{t}^{n}||^{1/n}	\right\rangle_{U} \approx \lim_{n \to \infty} \left(\langle || {\bf J}_{t}^{n}|| \rangle_{U} \right)^{1/n},
\end{align}

where $\langle ... \rangle_{U}$ indicates averaging over random connectivity matrices. In words, this assumption states that the spectral radius is stable between random realizations of the weights. Since
the Gelfand formula is valid for any matrix norm, we resort to
the sum of the singular values of $J_t^n$ : $\mu_n$ 

\begin{align}
\mu_{n} &= \left\langle \frac{1}{N} {\rm tr} ({\bf J}_{n}^{T} {\bf J}_{n})\right\rangle.
\end{align}

With this, we may then use the results of Appendix \ref{app:lyapunov-2}  to find that in the limit of large $N$, the spectral radius of the Jacobian is given by
\begin{align}
\langle \rho({\bf J}_{t}) \rangle = \lim_{n \to \infty}\left( \mu_{n}\right)^{\frac{1}{2n}}= \sigma(b_{z}) + (1 - \sigma(b_{z})) \rho_{t} \label{eq:gelfand-radius}	.
\end{align}
This expression also follows immediately from the equation for the spectral support when $a_{z} = 0$
\begin{align}
|\lambda - \sigma(b_{z})|^{2} = \rho_{t}^{2} (1 - \sigma(b_{z}))^{2}.
\end{align}

Through its effect on $\rho_{t}$, the reset gate will directly influence the spectral radius. To see this, recall again the expression for $\rho_{t}$
\begin{align}
\rho_{t}^{2} = a_{h}^{2} C_{\phi'} \left( C_{r} + a_{r}^{2}C_{r'} C_{h}\right),	
\end{align}
where 
\begin{align}
C_{\phi'} &= \int \left(\phi'\left( \sqrt{C_{y}} x\right)\right)^{2} Dx, \quad C_{y}  = C_{r} C_{h}\\
C_{r} & = \int Dx \left(\sigma(a_{r} \sqrt{C_{h}} x)\right)^{2}, \quad C_{r'}  = \int Dx \left(\sigma'(a_{r} \sqrt{C_{h}} x )\right)^{2}.
\end{align}
We have taken $b_{r} = 0$ for simplicity. We would like to examine the growth of $\rho_{t}$ as a function of $a_{r}$ at fixed $C_{h}$. The mean-field theory shows that since $C_{h}$ only depends on $a_{r}$ through the reset gate, it is not possible that the steady state distribution $C_{h}$ will grow without bound. Thus, $C_{h}$ must be bounded above and below by constant which do not grow with $a_{r}$. Next, $C_{r}$ is a bounded function of $a_{r}$, and for fixed $C_{h}$ $\frac{1}{4} \le C_{r} \le \frac{1}{2}$. Therefore, we may bound the spectral shaping parameter from below
\begin{align}
 a_{h}^{2} C_{\phi'} \left( \frac{1}{4} + a_{r}^{2} C_{r'} C_{h}\right) \le \rho_{t}^{2}	\le  a_{h}^{2} C_{\phi'} \left( \frac{1}{2} + a_{r}^{2} C_{r'} C_{h}\right).
\end{align}
Furthermore, since $C_{y} = C_{r} C_{h}$, we have that $C_{y}$ is similarly bounded for fixed $C_{h}$. An upper-bound on $C_{y}$ implies a non-zero lower bound on $C_{\phi'}$ which we call $c$: $c\le C_{\phi'}\le 1$. Finally, we come to $C_{r'}$, which tends to zero with increasing $a_{r}$. Thus, it remains to show that this tendency does not overwhelm the $a_{r}^{2}$ prefactor. To see this, we may develop an asymptotic expansion of the integral 
\begin{align}
C_{r'} & = \int Dx \left(\sigma'(a_{r} \sqrt{C_{h}} x)\right)^{2} = \frac{1}{6 a_{r} \sqrt{2\pi C_{h}}} + O(a_{r}^{-2}).
\end{align}

Combining this with the bounds we have argued for, we get

\begin{align}
a_{h}^{2} c \left( \frac{1}{4} + \frac{a_{r}}{6 \sqrt{2\pi}} \sqrt{ C_{h}} \right) \le \rho_{t} \le 	a_{h}^{2}  \left( \frac{1}{4} + \frac{a_{r}}{6 \sqrt{2\pi}} \sqrt{ C_{h}} \right) ,
\end{align}

and thus $\rho_{t} = \Theta(\sqrt{a_{r}})$.


\section{GRU Fixed Point Phase diagram}\label{app:gru-fixed-point-mft}

Here we describe the fixed point distributions implied by the implicit equations (\ref{eq:FP1}). We assume $v_{h} = v_{r} = b_{h} = 0$, but keep nonzero $b_{r}$. For fixed $a_{h}$ and $a_{r}$, and $a_{r}^{2}C_{y}<<1$, we expand the correlation functions

\begin{align}
C_{\phi} &=  a_{h}^{2} C_{y}	 - 2 a_{h}^{4} C_{y}^{2} + \frac{17}{3} a_{h}^{6} C_{y}^{3} + O(C_{y}^{4}) , \quad C_{r}  = c_{1} + c_{2} a_{r}^{2} C_{h} + c_{3} a_{r}^{4} C_{h}^{2} + O(C_{h}^{3}) ,\\
c_{1} & = \sigma(b_{r})^{2}, \quad c_{2} =   - \frac{( e^{-b_{r}} - 2e^{- 2b_{r}} )}{(1 + e^{ - b_{r}})^{4}} , \quad c_{3}  = - \frac{e^{2 b_{r}} ( e^{3 b_{r}} - 18 e^{ 2 b_{r}} + 33 e^{ b_{r}} - 8)}{4 (1 + e^{b_{r}})^{6}} .
\end{align}
Then using $C_{y} = C_{h} C_{r}$, the implicit equation (\ref{eq:FP1}) becomes 

\begin{align}
C_{h} =  a_{h}^{2} c_{1} C_{h} + \left( a_{h}^{2} a_{r}^{2} c_{2} - 2 a_{h}^{4} c_{1}^{2}\right) C_{h}^{2} + \left( c_{3} a_{h}^{2} a_{r}^{4} - 4 c_{1}c_{2} a_{h}^{4} a_{r}^{2} + \frac{17}{3} a_{h}^{6} c_{1}^{3}\right) C_{h}^{3}  + O(C_{h}^{4}) .
\end{align}
Setting $b_{r} = 0$ gives
\begin{align}
 1 = 	\frac{a_{h}^{2}}{4} +  \frac{a_{h}^{2}}{16}  \left(  a_{r}^{2}  - 2 a_{h}^{2} \right)  C_{h} + \frac{a_{h}^{2} }{192} ( 17a_{h}^{4}  - 12 a_{h}^{2} a_{r}^{2} - 6 a_{r}^{4}) C_{h}^{2} + O(C_{h}^{3}) .
\end{align}
A straightforward analysis of this expression leads to the perturbative solutions quoted in the main text, with the function 
\begin{align}
f(a_{r}) = 4 ( 31552 - 11424 a_{r}^{2} + 744 a_{r}^{4} + 72 a_{r}^{6}  + 9 a_{r}^{8})/( 3 a_{r}^{4} + 24 a_{r}^{2} - 136)^{2}, \label{eq:fpoly}
\end{align}

which is positive for $a_{r}> \sqrt{8}$. 

We comment briefly on the effects of finite bias $b_{r}$. The critical line past which only a single nonzero fixed point exists moves to $a_{h} = 1 + e^{ - b_{r}}$. For a finite bias, there is still only a single zero fixed point for $a_{h} < \sqrt{2}$.


\section{GRU : Singular values of the long-term Jacobian} \label{app:lyapunov-2}

In this section, we provide results on the moments of the
singular values of the long-term Jacobian for fixed points in two
limiting cases.

\paragraph{Binary Update Gate}
We first consider the Jacobian for fixed points with update gates which are switch-like ($a_z \rightarrow \infty$). In this case, ${\bf z}_{t}$ is vectors of binary variables with distribution $P(z) = \alpha \delta(z) + (1 - \alpha) \delta(z - 1)$, where $\alpha$ depends on $b_{z}$ (see Eq.(\ref{eq:alpha})) and ${\bf z}_{t}' = 0$, so that the Jacobian is
\begin{align}
{\bf J} = {\bf \hat{z}} + (\mathbbm{1} - {\bf \hat z}) \hat{\phi}'({\bf y}) U_{h}( \hat{r} + \hat{r}' \hat{h} U_{r}).
\end{align}

Since ${\bf \hat{z}}( \mathbbm{1} - {\bf \hat z}) = 0$, and ${\bf z}^{k} = {\bf z}$, we have for the late-time Jacobian

\begin{align}
{\bf J}_{n} = 	{\bf J}^{n} = \sum_{q = 0}^{n} \left((\mathbbm{1} - {\bf \hat z}) \hat{\phi}'({\bf y})U_{h}( \hat{r} + \hat{r}' \hat{h} U_{r})\right)^{q} {\bf \hat{z}} .
\end{align}

This allows us to obtain the mean square of the singular values to leading order in $N$

\begin{align}
\mu_{n} &= \left\langle \frac{1}{N} {\rm tr} ({\bf J}_{n}^{T} {\bf J}_{n})\right\rangle  = (1- \alpha) \sum_{k = 0}^{n}\left( \alpha \rho^{2}\right)^{k} + O(N^{-1}).
\end{align}

Here, we have also taken the expectation over $z$, using for instance,

\begin{align}
\frac{1}{N} {\rm tr} \left( (\mathbbm{1} - {\bf \hat z})^{2}  (\hat{\phi}'({\bf y}))^{2} \right) &\approx \mathbbm{E}[ (1 - z)^{2} (\phi'(\eta))^{2}] =\alpha   \mathbbm{E}[ (\phi'(\eta))^{2}] .
\end{align}

For $\alpha \rho^{2} < 1$, the moment converges to

\begin{align}
\mu_{n} \to  \frac{(1 - \alpha)}{1 - \alpha \rho^{2}}.
\end{align}

\paragraph{Constant Update Gate}
In the limit of zero $a_{z}$, the behavior is quite different. Then ${\bf \hat{z}}$ becomes a constant diagonal matrix with elements $\sigma(b_{z})$. The late time Jacobian is

\begin{align}
{\bf J}_{n} = \sum_{q= 0}^{n} \left( { n \atop q}\right) {\bf \hat{z}}^{n-q} \left( (1 - {\bf \hat{z}}) \hat{\phi}'({\bf y}) U_{h}\left( {\bf \hat r} + {\bf \hat r}' {\bf \hat h} U_{r}\right)\right)^{q}.
\end{align}

The mean squared singular values can be computed again in the limit of large $N$ to yield

\begin{align}
\mu_{n} = \sum_{q = 0}^{n} 	\left( { n \atop q}\right)^{2}  \sigma(b_{z})^{2(n-q)}  (1 - \sigma(b_{z}))^{2q} \rho^{2 q}, \quad \rho^{2} = a_{h}^{2} \sigma(b_{r}) \mathbbm{E}[ (\phi(y)')^{2}] \left( \mathbbm{E}[ r^{2}] + a_{r}^{2}\mathbbm{E}[ (r')^{2} h^{2}]\right).
\end{align}

Formally, this can be expressed in terms of the Gauss hypergeometric function

\begin{align}
\mu_{n} = 	\sigma(b_{z})^{2n} {}_{1}F_{2} \left( - n, - n; -1; \frac{(1 - \sigma(b_{z}))^{2} \rho^{2}}{\sigma(b_{z})^{2}} \right).
\end{align}

Another representation which is useful for asymptotic analysis is

\begin{align}
\mu_{n} = \int_{0}^{2\pi} \frac{d \theta}{2\pi} \left| \sigma(b_{z}) + (1 - \sigma(b_{z})) \rho e^{ i \theta}\right|^{2n}	.
\end{align}

This can be evaluated by a saddle-point argument for large $n$ to give the spectral radius in (\ref{eq:gelfand-radius}).


\section{LSTM: Proof of Theorem (\ref{thm:LSTM-spec-curve})} \label{app:lstm-rmt}

Linearization of the LSTM Jacobian Eq.(\ref{eq:lstm-jacobian}) is achieved by  (dropping the carat $\,\,\hat{}\,\,$ with the understanding that all boldfaced variables are diagonal matrices)

\begin{align}
{\bf M}_{t}= \left( \begin{array}{cccccc}
 	{\bf f}_{t}  & 0  & {\bf i}_{t} \phi'({\bf y}_{t}) & 0  & \phi({\bf y}_{t}) & {\bf c}_{t-1}\\
 	{\bf m}_{t} {\bf f}_{t} & 0 & {\bf m}_{t} {\bf i}_{t} \phi'({\bf y}_{t}) & \phi({\bf c}_t) & {\bf m}_{t} \phi({\bf y}_{t}) & {\bf m}_{t} {\bf c}_{t-1} \\
 	0 & U_{h} & 0 & 0 & 0 & 0 \\
 	0 & {\bf o}_{t}' U_{o} & 0 & 0 & 0 & 0 \\
 	0 & {\bf i}_{t}' U_{i} & 0 & 0 & 0 & 0 \\
 	0 & {\bf f}_{t}' U_{f} & 0 & 0 & 0 & 0 
 \end{array}\right).
\end{align}
We express this as ${\bf M}_{t} = \hat{A} + \hat{B} \hat{U} \hat{C}$, where

\begin{align*}
\hat{A} = 	\left( \begin{array}{cccccc}
 	{\bf f}_{t}  & 0  & {\bf i}_{t} \phi'({\bf y}_{t}) & 0  & \phi({\bf y}_{t}) & {\bf c}_{t-1}\\
 	{\bf m}_{t} {\bf f}_{t} & 0 & {\bf m}_{t} {\bf i}_{t} \phi'({\bf y}_{t}) & \phi({\bf c}_t) & {\bf m}_{t} \phi({\bf y}_{t}) & {\bf m}_{t} {\bf c}_{t-1} \\
 	0 & 0 & 0 & 0 & 0 & 0 \\
 	0 & 0 & 0 & 0 & 0 & 0 \\
 	0 & 0 & 0 & 0 & 0 & 0 \\
 	0 & 0 & 0 & 0 & 0 & 0 
 \end{array}\right) , \quad \hat{C} =  \left( \begin{array}{cccccc}
 	0  & 0  & 0 & 0  & 0 & 0\\
 	0 & 0 & 0 & 0 & 0 & 0 \\
 	0 & \mathbbm{1} & 0 & 0 & 0 & 0 \\
 	0 & \mathbbm{1} & 0 & 0 & 0 & 0 \\
 	0 & \mathbbm{1} & 0 & 0 & 0 & 0 \\
 	0 & \mathbbm{1} & 0 & 0 & 0 & 0 
 \end{array}\right),\\
\end{align*}
$\hat{U} = {\rm b diag} \left( 0 , 0 , U_{h}, U_{o}, U_{i}, U_{f}\right)$, and $\hat{B} = {\rm b diag} \left( 0, 0, \mathbbm{1}, {\bf o}_{t}', {\bf i}_{t}', {\bf f}_{t}'\right)$. We denote by ${\rm bdiag}(A_{1},..., A_{n})$ a block diagonal matrix whose blocks $A_{i}$ are $N\times N$ matrices. 

The generalized eigenvalue problem is ${\bf M}_{t} {\bf v} = I_{\lambda, 2} {\bf v}$, where  $I_{\lambda, 2} = {\rm bdiag}(\lambda, \lambda, 1, 1, 1, 1)$ and $\lambda$ are eigenvalues of the Jacobian ${\bf J}_{t}$. As before, we define the generalized resolvent

\begin{align}
{\bf G} = \left( I_{\lambda,2} - {\bf M}_{t}\right)^{-1},	
\end{align}

which is a $6 N \times 6N$ matrix. We describe this matrix by $N\times N$ subblocks ${\bf G}_{ab}$, for $a,b = 1,.., 6$. The resolvent of the Jacobian is then found in the upper left corner of the generalized resolvent

\begin{align}
G = \left( \lambda  - {\bf J}_{t}\right)^{-1} = \left( \begin{array}{cc}
 {\bf G}_{11} & {\bf G}_{12}\\
 {\bf G}_{21} & {\bf G}_{22}	
 \end{array}\right)	
\end{align}

Following exactly the hermitization procedure outlined in Sec. (\ref{app:gru-rmt}), we define $F(\lambda, \bar{\lambda}) = f_{99} ( f_{33} + f_{44} + f_{55} + f_{66})$, and obtain the corresponding set of equations from the Dyson equation. To express these in compact form, let

\begin{align}
{\bf Q}_{t}(\lambda) & = |\lambda - {\bf f}_{t}|^{2} {\bf q}_{t} + |\lambda|^{2} {\bf p}_{t}, \quad {\rm where}\\
{\bf q}_{t} & = a_{o}^{2}({\bf o}'_{t})^{2} \phi({\bf c}_{t})^{2},\quad  {\bf p}_{t} = {\bf o}_{t}^{2} (\phi'({\bf c}_{t}))^{2} \left( a_{f}^{2}{\bf c}_{t-1}^{2} ({\bf f}_{t}')^{2} + a_{i}^{2}({\bf i}_{t}')^{2} \phi({\bf y}_{t})^{2} + a_{h}^{2}{\bf i}_{t}^{2} (\phi'({\bf y}_{t}))^{2}\right)
\end{align}

Then we find the set of equations for the trace of the resolvent $G(\lambda)$ and the auxiliary function $F(\lambda)$

\begin{align}
G(\lambda) &= \frac{1}{N} {\rm tr} \left[ \frac{	|\lambda|^{2} (\bar{\lambda} - {\bf \hat f}_{t}) + \bar{\lambda} |\lambda - {\bf \hat f}_{t}|^{2} + F(\lambda, \bar{\lambda}) \partial_{\lambda} {\bf \hat Q}_{t}(\lambda, \bar{\lambda}) }{ |\lambda|^{2} |\lambda - {\bf \hat f}_{t}|^{2} + F(\lambda, \bar{\lambda}) {\bf \hat Q}_{t}(\lambda, \bar{\lambda})} \right]
F(\lambda) &= \frac{1}{N} {\rm tr} \left[ \frac{F(\lambda){\bf \hat Q}_{t}(\lambda) }{|\lambda|^{2}|\lambda - {\bf \hat f}_{t}|^{2} + F(\lambda)  {\bf \hat Q}_{t}(\lambda	)}\right]\label{eq:lstm-dyson2}
\end{align}

Outside the support of the eigenvalue density, the resolvent is holomorphic. We see that a consistent solution to these equations has $F(\lambda) = 0$, in which case $G(\lambda)$ simplies to the result quoted in (\ref{eq:lstm-G}), which is holomorphic. We conclude that $F = 0$ defines the exterior of the support. Inside the spectral support, we find the implicit equation can be solved for $F> 0$, which leads to (\ref{eq:lstm-F}). Continuity of $G(\lambda)$ implies that at the boundary of the spectral support, the holomorphic solution must match the non-analytic solution. The curve for the boundary support then follows by considering (\ref{eq:lstm-dyson2}) $\lambda$ approaching the boundary from the interior. This limit allows us to divide through by $F$ first, and then set $F = 0$ to finally find the boundary curve in Corollary (\ref{cor:lstm-spec-curve}). So far we have not assumed the mean field theory, and keep an explicit sum over neurons. In the mean field theory, the neurons (hidden and cell states) describe independent stochastic processes, and allows us to replace the trace with an expectation value over the effective stochastic variables for each site.

We use the mean field theory now to describe the density of eigenvalues which is obtained from the resolvent. Using the definition (\ref{eq:density_resolvent}), we find

\begin{align*}
\mu (\lambda) =\frac{1}{\pi} \mathbbm{E}\left[ \frac{  F(\lambda) (|\lambda - f_{t}|^{4} q_{t} + |\lambda|^{4} p_{t}) - \partial_{\bar{\lambda}} F(\lambda) \left( \bar{\lambda} |\lambda - f_{t}|^{4} q_{t} + (\bar{\lambda} - f_{t}) |\lambda|^{4} p_{t}\right) + F(\lambda)^{2}  f_{t}^{2} p_{t} q_{t}}{\Big(|\lambda|^{2} |\lambda - f_{t}|^{2} + F(\lambda) Q_{t}\Big)^{2}}\right].
\end{align*}

To find $\partial_{\bar{\lambda}}F$, we differentiate the implicit equation for $F$,
\begin{align*}
\partial_{\bar{\lambda}} F  = -  \mathbbm{E} \left[ \frac{\lambda |\lambda - f_{t}|^{4} q_{t} + (\lambda - f_{t}) |\lambda|^{4} p_{t}}{ \left( |\lambda|^{2} |\lambda - f_{t}|^{2} + F Q_{t}\right)^{2}}\right] \left( \mathbbm{E} \left[ \frac{Q_{t}^{2}}{\left( |\lambda|^{2} |\lambda - f_{t}|^{2} + FQ_{t}\right)^{2}} \right]\right)^{-1}	
\end{align*}

\subsection{Proof of Corollaries (\ref{cor:lstm-zero-fp}) and (\ref{cor:lstm-zero-fp-stability}): Biases Only}

The equations all simplify considerably when $a_{k} = 0$ for $k \in \{ f, i, o\}$. In this case, $k_{t} = k = \sigma(b_{k})$. We get $p_{t}  =a_{h}^{2} o^{2} i^{2} (\phi'(c_{t}))^{2} (\phi'(y_{t}))^{2}$, the spectral support

\begin{align}
  \Sigma:=& \{ \lambda \in\mathbbm{C}: |\lambda - f|^{2} = \mathbbm{E} \left[ \, p_{t}\right] \}, 
\end{align}

and the resolvent inside the support

\begin{align}
G(\lambda)&  = \frac{1}{\lambda} + \mathbbm{E}\left[ \frac{(\bar{\lambda} - f) }{ |\lambda - f|^{2} + F(\lambda) p_{t}} \right]\\
1 &= \mathbbm{E} \left[ \frac{  p_{t}}{|\lambda - f|^{2} + F(\lambda)  p_{t}}\right].
\end{align}

These become particular simple at the zero FP when $c = h = 0$. In this case, we have $p = \sigma(b_{o})^{2} \sigma(b_{i})^{2} a_{h}^{2}$, and straightforward substitution into the equations above gives Corollaries (\ref{cor:lstm-zero-fp}) and (\ref{cor:lstm-zero-fp-stability}).

\subsection{Only Forget}\label{app:lstm-only-forget}

{ 

Here we examine the effects of the forget gate by setting $a_{o} = a_{i} = 0$. This gives $q_{t} = 0$, and $p_{t}=  o_{t}^{2} (\phi'(c_{t}))^{2} \left( a_{f}^{2}c_{t-1}^{2} (f_{t}')^{2} + a_{h}^{2}i^{2} (\phi'(y_{t}))^{2}\right)$, and the resolvent inside the spectral support
\begin{align}
G(\lambda) &= \frac{1}{\lambda} + \mathbbm{E}\left[ \frac{(\bar{\lambda} - f_{t}) }{ |\lambda - f_{t}|^{2} + F(\lambda) p_{t}  } \right]\\
1 &= \mathbbm{E}\left[ \frac{  p_{t}}{|\lambda - f_{t}|^{2} + F  p_{t}}\right].
\end{align}

The spectral boundary curve is

\begin{align}
1 &= \mathbbm{E}\left[ \frac{ o_{t}^{2} (\phi'(c_{t}))^{2} \left( a_{f}^{2}c_{t-1}^{2} (f_{t}')^{2} + a_{h}^{2} i^{2} (\phi'(y_{t}))^{2}\right)	}{|\lambda - f_{t}|^{2}}\right]	\label{eq:lstm-spec-curve-forget}
\end{align}

In the limit when $a_{f} = \infty$, we must take care with taking expectation values since in the mean-field limit, the distributions of $c$ and $f$ are not separable. We proceed with reasonable arguments. To arrive at explicit formulas, we assume a hard-tanh activation function $\phi$. When $a_{f} = \infty$, $f$ becomes a binary variable. We may then consider two cases 

\begin{align}
c_{t} \sim \begin{cases}
 c_{t-1} + i_{t} \phi(y_{t}),\quad &f_{t} = 1\\
 i_{t} \phi(y_{t})	, \quad &f_{t} = 0
 \end{cases}
\end{align}

In the first instance, $c_{t}$ exhibits fast growth that overwhelms the finite noise term $\phi(y_{t})$, and in this case we approximate the distributions of $c_{t}$ and $y_{t}$ as independent. In the latter case, $c_{t}$ is controlled exclusively by the fluctuations of $\phi(y_{t})$, and far from being independent, they are slaved. Defining for ease of notation $\tilde{F} = a_{h}^{2}\sigma(b_{0})^{2} \sigma(b_{i})^{2} F$, we get the implicit equation for $F$:

\begin{align}
1 =  \frac{1}{2} \mathbbm{E}\left[ \frac{a_{h}^{2}o^{2} i^{2} \phi'(c)^{2} \phi'(y)^{2}}{|\lambda|^{2} + \tilde{F} \phi'(c_{t})^{2} \phi'(y_{t})^{2}}\right]_{P(c, y|f = 0)} +\frac{1}{2} \mathbbm{E} \left[ \frac{a_{h}^{2}o^{2} i^{2} \phi'(c_{t})^{2} \phi'(y_{t})^{2}}{|\lambda - 1|^{2} + \tilde{F} \phi'(c_{t})^{2} \phi'(y_{t})^{2}}\right]_{P(c,y|f = 1)}	
\end{align}

When $f = 1$, we assume $\phi'(c) = 0$, since the activation saturates for large $c$. Therefore, we are left with the first case, when $c_{t} \sim \sigma(b_{i}) \phi(y)$ and use the Gaussianity of $y_{t}$ to evaluate

\begin{align}
1 &=  \frac{1}{2} \mathbbm{E}\left[ \frac{a_{h}^{2}o^{2} i^{2} [\phi'(\sigma(b_{i}) \phi(y))]^{2} \phi'(y)^{2}}{|\lambda|^{2} + \tilde{F} [\phi'(\sigma(b_{i}) \phi(y))]^{2} \phi'(y)^{2}}\right],\\
& =\frac{o^{2} i^{2} a_{h}^{2} \eta}{2} \frac{1}{|\lambda|^{2} + i^{2} o^{2} a_{h}^{2} F}, \quad \eta = \int_{-1/a_{h}}^{1/a_{h}} \frac{e^{ - y^{2}/2C_{h}}}{\sqrt{2\pi C_{h}}} .
\end{align}

Using this we may evaluate the resolvent

\begin{align}
G(\lambda) &= \frac{1}{\lambda} + \frac{1}{2} \mathbbm{E}\left[ \frac{\bar{\lambda}}{|\lambda|^{2} + \tilde{F} \phi'(c_{t})^{2} \phi'(y_{t})^{2}}\right]_{P(c, y|f = 0)} +\frac{1}{2} \mathbbm{E} \left[ \frac{\bar{\lambda} - 1}{|\lambda - 1|^{2} + \tilde{F} \phi'(c_{t})^{2} \phi'(y_{t})^{2}}\right]_{P(c,y|f = 1)}	\\
& \approx  \frac{1}{\lambda}  + \frac{(1 - \eta)}{2} \frac{1}{\lambda} + \frac{\bar{\lambda}}{o^{2} i^{2} a_{h}^{2}} + \frac{1}{2} \frac{1}{\lambda - 1}  .
\end{align}
Then using Eq.(\ref{eq:density_resolvent}) results in Eq. (\ref{eq:lstm-density-binary-forget}).

\paragraph{Effects on Spectral Radius}
The previous section assumes switch-like forget gate. Here we relax this assumption and consider the growth of the spectral radius with $a_{f}$. We focus on the term involving gradients of $f_{t}$ in the equation for the spectral boundary curve (\ref{eq:lstm-spec-curve-forget}). Since $f_{t}'$ will be sharply peak around $\eta_{t}^{f} = 0$ (for zero bias), we treat it like a delta functional and find this contribution behaves approximately like

\begin{align}
    \frac{a_{f}}{6 \sqrt{2\pi C_{h}}} \frac{1}{|\lambda - 1/2|^{2}}  \mathbbm{E} \left[(\phi'(c_{t})^{2}) c_{t-1}^{2} \right]_{P(c, y| f = 1/2)}
\end{align}

When $f = 1/2$, the update equation for the cell state is $c_{t} = (1/2) c_{t-1} + i \phi(y_{t})$, and due to the fast decay of $c_{t}$ we argue that $\phi'(c_{t}) \approx 1$. Combining this with the results in the previous section gives for the spectral boundary curve
\begin{align}
    1 = \frac{a_{f} \mathbbm{E}\left[ c_{t-1}^{2}\right]_{P(c|f=1/2)}}{6 \sqrt{2\pi C_{h}} |\lambda - 1/2|^{2}} + \frac{o^{2} i^{2} a_{h}^{2} \eta}{2 |\lambda|^{2}}
\end{align}

Assuming of course that $C_{h} > 0$, which is true for networks with nonzero activity. We also need to make some reasonable assumption about the behavior of the correlation function which appears in this expression. For large $a_{f}$, the cell state will generally grow without bound for $f \approx 1$. However, conditioned on $f = 1/2$, we may argue that the autocorrelation remains bounded with increasing $a_{f}$. Together, these assumptions show that for increasing $a_{f}$, the first term takes over and the spectral radius grows like $\sqrt{a_{f}}$. 

This poses an apparent contradiction: the spectral radius would appear to grow with $a_{f}$, yet in the limit $a_{f} = \infty$ it is independent of $a_{f}$ according to (\ref{eq:lstm-density-binary-forget}). This likely occurs in the following way: the conditional probability used to evaluate the autocorrelation function must approach measure zero as $a_{f} \to \infty$. Therefore, we conclude that the spectral radius does not grow monotonically with $a_{f}$. In Fig.(\ref{fig:lstm-plots}), we manage to capture this initial growth, which we conjecture is followed by a suppression for sufficiently large $a_{f}$.

\subsection{Only Input}\label{app:lstm-only-input}

Now we isolate the input gate's effect on the Jacobian by setting $a_{f} = a_{o} = 0$, which makes $q_{t} = 0$, and $p_{t} = o^{2} \phi'(c_{t})^{2}\left( a_{i}^{2} i_{t}'^{2} \phi(y_{t})^{2} + a_{h}^{2} i_{t}^{2} \phi'(y_{t})^{2}\right)$. using Cor. (\ref{cor:lstm-spec-curve}) to write spectral boundary curve, we may immediately extract the spectral radius 
\begin{align}
\rho({\bf J}_{t}) = \sigma(b_{f}) + \sqrt{\mathbbm{E}\left[ p_{t}\right]}, \quad 	p_{t}  = o^{2} (\phi'(c_{t}))^{2} \left(  a_{i}^{2} (i_{t}')^{2} \phi(y_{t})^{2} + a_{h}^{2}i_{t}^{2} (\phi'(y_{t}))^{2}\right).
\end{align}

To obtain the asymptotic scaling of the spectral radius with increasing $a_{i}$, we may use the Cauchy-Schwartz inequality to  establish the bound 

\begin{align}
\mathbbm{E}\left[ p_{t}\right] \le a_{i}^{2}o^{2} \mathbbm{E}\left[ (i_{t}')^{2}\right] \mathbbm{E}\left[ (\phi'(c_{t}))^{2} \phi(y_{t})^{2}\right] + a_{h}^{2} o^{2} \mathbbm{E}\left[ i_{t}^{2}\right] \mathbbm{E}\left[ (\phi'(c_{t}))^{2} \phi'(y_{t})^{2}\right].
\end{align}

First of all, $\phi'(c_{t})$, $\phi(y_{t})$ and $\phi'(y_{t})$ are all bounded functions, which implies that their correlations cannot grow with $a_{i}$. Furthermore, $\mathbbm{E}[i_{t}^{2}]\le 1/2$, so we are left with $\mathbbm{E}[ (i_{t}')^{2}]$. Using the same arguments as in Sec. \ref{app:gru-reset-spectral-radius}, we find

\begin{align}
\mathbbm{E}\left[ (i_{t}')^{2}\right] = 	\frac{1}{6 a_{i} \sqrt{2\pi C_{h}}} + O(a_{i}^{-2}).
\end{align}

Combining all of this, we arrive at our main result

\begin{align}
\rho({\bf J}_{t})  = O(\sqrt{a_{i}})	.
\end{align}

A further assumption that $\mathbbm{E}\left[ (i_{t}')^{2} \phi'(c_{t})^{2}\right] = \Theta \left( \mathbbm{E}\left[ (i_{t}')^{2} \right] \mathbbm{E}\left[\phi'(c_{t})^{2}\right]\right)$ would allow us to also obtain a lower bound that would imply $\rho({\bf J}_{t}) = \Theta(\sqrt{a_{i}})$. Our numerical experiments indicate that such a lower bound should hold.

\subsection{Only Output}\label{app:lstm-only-output}

Here we study the effects of the output gate on the spectral radius. Fixing the other gates by setting $a_{f} = a_{i} = 0$, using Cor. (\ref{cor:lstm-spec-curve}) the spectral boundary curve becomes
\begin{align}
1 &= \frac{\mathbbm{E}[ q_{t}]}{|\lambda|^{2}}  +\frac{\mathbbm{E}[ p_{t}]}{|\lambda - \sigma(b_{f})|^{2}} , \quad q_{t} = a_{o}^{2}(o_{t}')^{2} \phi(c_{t})^{2}, \quad p_{t}  = a_{h}^{2} o_{t}^{2}  i^{2} (\phi'(c_{t}))^{2}(\phi'(y_{t}))^{2}
\end{align}
Since $o_{t}^{2}$ is bounded, $p_{t}$ will remain bounded as $a_{o}$ is increased. We can thus make the obvious bounds 
\begin{align}
0 \le \mathbbm{E}\left[ p_{t}\right] \le \frac{1}{2} \sigma(b_{i})^{2} a_{h}^{2}	,
\end{align}
which is just to show that this factor will not grow with $a_{o}$. Next, note that in the mean field limit for the LSTM network, the cell state $c_{t}$ is {\it independent} of $o_{t}$, and therefore we may write
\begin{align}
\mathbbm{E}\left[ q_{t}\right] = a_{o}^{2} \mathbbm{E}[ (o_{t}')^{2}] \mathbbm{E}\left[ (\phi(c_{t})^{2}\right]	
\end{align}
Now following the same arguments as in Sec. \ref{app:gru-reset-spectral-radius}, we evaluate the correlation function
\begin{align}
\mathbbm{E}[ (o_{t}')^{2}] = \frac{1}{6 a_{o}\sqrt{2\pi C_{h}}} + O(a_{o}^{-2}).
\end{align}
Therefore, we find
\begin{align}
\mathbbm{E}[ q_{t}] \sim a_{o} \frac{\mathbbm{E}\left[ \phi(c_{t})^{2}\right]	}{6 \sqrt{2\pi C_{h}}} 
\end{align}
Which means that for large $a_{o}$, the spectral radius grows like $\rho({\bf J}_{t})	= \Theta (\sqrt{a_{o}})$.

}


\section{GRU : Training on sequential MNIST} \label{app:train_seq_MNIST}
Here we provide details of training a GRU on the sequential MNIST task \cite{jing2019gated, kerg2019non}.
Each 28x28 MNIST image is fed into a 256 dimensional GRU one pixel per timestep. The output of the GRU is used to decide the digit class at the end of the 784 timesteps. 
We also use zero biases for the GRU to make the training closer to the theory we have worked out -- although incorporating biases in the theory is a trivial extension.
The batch size chosen was 100 (for 60000 images in the training set), and the maximum number of epochs was taken to be 15.
For ease of discussion, we show results for a low value of $a_r=0.1$.


\begin{figure}[h]
\begin{centering}
\includegraphics[scale=0.6]{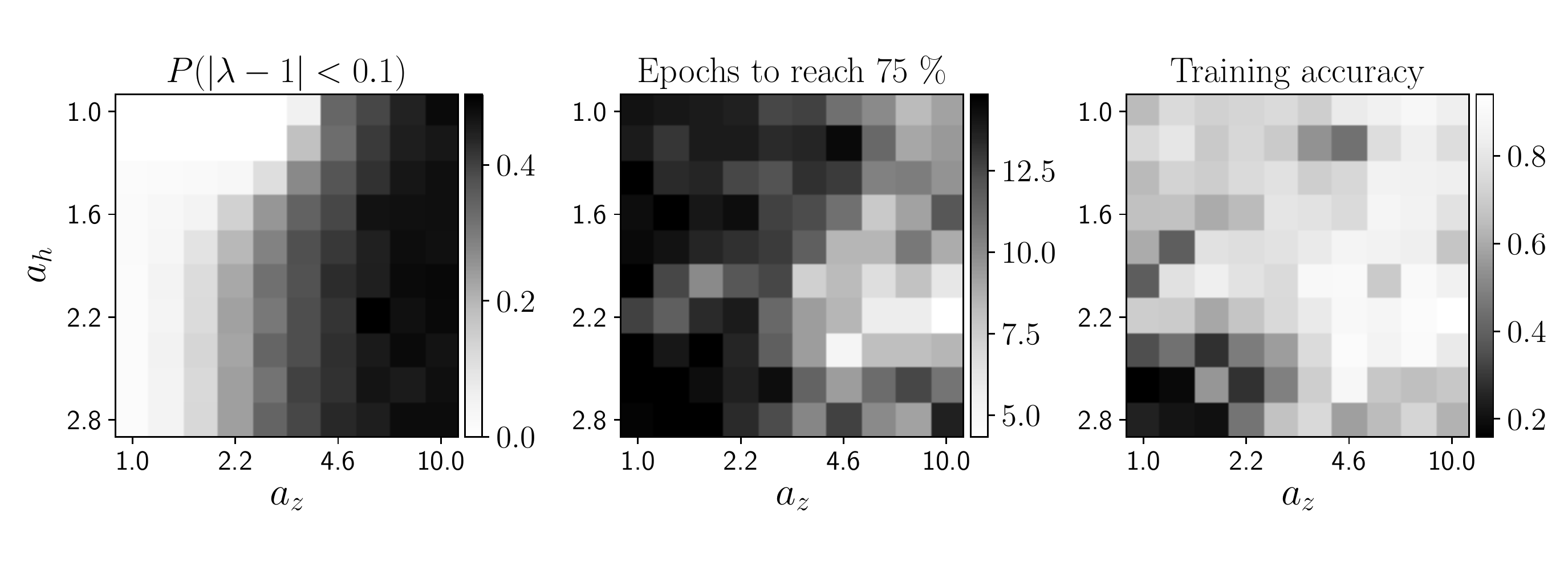}
\par\end{centering}
\caption{Training on sequential MNIST: ({\bf left}) 
The density of eigenvalues of the Jacobian near 1.0
as a function of $a_h$ and $a_z$.
({\bf middle}) 
The number of epochs required to reach a training 
accuracy of 75 \%
as a function of $a_h$ and $a_z$.
({\bf right}) 
The training accuracy after 15 epochs 
as a function of $a_h$ and $a_z$.} \label{fig:training_fig_1}
\end{figure}

Fig. \ref{fig:training_fig_1} shows the i) density of Jacobian eigenvalues
near 1.0 (left); ii) the average number of epochs required to reach
a training accuracy of 75 \% (middle) and iii) the average training accuracy (right) as a function of $a_h$ and $a_z$.
As we can see in Fig \ref{fig:training_fig_1} (left)
increasing $a_z$ leads to an increase in the density
of eigenvalues near 1.0. The training accuracy and the 
number of epochs needed to reach 75\% accuracy both improve with increasing $a_z$, however, the best values
-- especially for the training time -- seem to be near 
$a_h=2.0$ which is the critical value at which 
the zero fixed-point becomes unstable; benefits of 
training at the ``edge-of-chaos" has been noted in previous work \cite{Bertschinger2004,glorot2010understanding}. Training appears to be harder once we get further into the chaotic regime. This is likely related to the rate of growth of gradients in this regime, and will be investigated elsewhere.

\end{document}